\documentclass[11pt]{article}
\usepackage[margin=1in]{geometry}
\usepackage[numbers, compress]{natbib}

\usepackage[utf8]{inputenc} 
\usepackage[T1]{fontenc}    
\usepackage{hyperref}       
\usepackage{url}            
\usepackage{booktabs}       
\usepackage{amsfonts}       
\usepackage{nicefrac}       
\usepackage{microtype}      
\usepackage{xcolor}         
\usepackage{subcaption}
\usepackage[font=small,labelfont=bf]{caption}

\usepackage{yub}
\usepackage{smile}
\usepackage{enumitem}

\hypersetup{
    colorlinks,
    linkcolor={blue!50!black},
    citecolor={blue!50!black},
}
\colorlet{linkequation}{blue}

\def\shownotes{0}  
\ifnum\shownotes=1
\newcommand{\authnote}[2]{{\scriptsize $\ll$\textsf{#1 notes: #2}$\gg$}}
\else
\newcommand{\authnote}[2]{}
\fi

\def\bartau{{\bar \tau}}
\def\barlambda{{\bar \lambda}}
\def\barb{{\bar b}}
\def\bJ{{\boldsymbol J}}
\def\ball{{\mathsf B}}
\def\op{{\rm op}}
\def\oG{{\overline G}}
\def\de{{\rm d}}
\def\R{{\mathbb R}}

\newcommand{\coverage}{{\rm Coverage}}

\title{Understanding the Under-Coverage Bias \\
  in Uncertainty Estimation}

\author{%
  Yu Bai\thanks{Salesforce Research. E-mail:~\texttt{yu.bai@salesforce.com}}
  \and
  Song Mei\thanks{University of California, Berkeley. E-mail:~\texttt{songmei@berkeley.edu}}
  \and
  Huan Wang\thanks{Salesforce Research. E-mail:~\texttt{\{huan.wang, cxiong\}@salesforce.com}}
  \and
  Caiming Xiong\footnotemark[3]
}

\date{\today}

\begin{document}

\maketitle

\begin{abstract}
 
Estimating the data uncertainty in regression tasks is often done by learning a quantile function or a prediction interval of the true label conditioned on the input. It is frequently observed that quantile regression---a vanilla algorithm for learning quantiles with asymptotic guarantees---tends to \emph{under-cover} than the desired coverage level in reality. While various fixes have been proposed, a more fundamental understanding of why this under-coverage bias happens in the first place remains elusive.
  
In this paper, we present a rigorous theoretical study on the coverage of uncertainty estimation algorithms in learning quantiles. We prove that quantile regression suffers from an inherent under-coverage bias, in a vanilla setting where we learn a realizable linear quantile function and there is more data than parameters. More quantitatively, for $\alpha>0.5$ and small $d/n$, the $\alpha$-quantile learned by quantile regression roughly achieves coverage $\alpha - (\alpha-1/2)\cdot d/n$ regardless of the noise distribution, where $d$ is the input dimension and $n$ is the number of training data. Our theory reveals that this under-coverage bias stems from a certain high-dimensional parameter estimation error that is not implied by existing theories on quantile regression. Experiments on simulated and real data verify our theory and further illustrate the effect of various factors such as sample size and model capacity on the under-coverage bias in more practical setups.

\end{abstract}


  
\section{Introduction}
This paper is concerned with the problem of uncertainty estimation in regression problems. Uncertainty estimation is an increasingly important task in modern machine learning applications---Models should not only make high-accuracy predictions, but also have a sense of how much the true label may deviate from the prediction. This capability is crucial for deploying machine learning in the real world, in particular in risk-sensitive domains such as medical AI~\citep{begoli2019need,jiang2012calibrating}, self-driving cars~\citep{michelmore2018evaluating}, and so on. A common approach for uncertainty estimation in regression is to learn a \emph{quantile function} or a \emph{prediction interval} of the true label conditioned on the input, which provides useful distributional information about the label. Such learned quantiles are typically evaluated by their \emph{coverage}, i.e., probability that it covers the true label on a new test example. For example, a learned $90\%$ upper quantile function should be an actual upper bound of the true label at least $90\%$ of the time.


Algorithms for learning quantiles date back to the classical quantile regression~\citep{koenker2001quantile}, which estimates the quantile function by solving an empirical risk minimization problem with a suitable loss function that depends on the desired quantile level $\alpha$. Quantile regression is conceptually simple, and is theoretically shown to achieve asymptotically correct coverage as the sample size goes to infinity~\citep{koenker1978regression} or approximately correct coverage in finite samples under specific modeling assumptions~\citep{meinshausen2006quantile,takeuchi2006nonparametric,steinwart2011estimating}. However, it is observed that quantile regression often \emph{under-covers} than the desired coverage level in practice~\citep{romano2019conformalized}. Various alternative approaches for constructing quantiles and confidence intervals are proposed in more recent work, for example by aggregating multiple predictions using Bayesian neural networks or ensembles~\citep{gal2016dropout,lakshminarayanan2016simple}, or by building on the conformal prediction technique to construct prediction intervals with finite-sample coverage guarantees~\citep{vovk2005algorithmic,vovk2012conditional,lei2018distribution,romano2019conformalized}. However, despite these advances, a more fundamental understanding on why vanilla quantile regression exhibits this under-coverage bias is still lacking. 


This paper revisits quantile regression and presents a first precise theoretical study on its coverage, in a new regime where the number of samples $n$ is proportional to the dimension $d$, and the ratio $d/n$ is small (so that the problem is under-parametrized). Our main result shows that quantile regression exhibits an inherent under-cover bias under this regime, even in the well-specified setting of learning a linear quantile function when the true data distribution follows a Gaussian linear model. To the best of our knowledge, this is the first rigorous theoretical justification of the under-coverage bias. Our main contributions are summarized as follows.
\begin{itemize}[leftmargin=1.5pc]
\item \emph{We prove that linear quantile regression exhibits an inherent under-coverage bias} in the well-specified setting where the data is generated from a Gaussian linear model, and the number of samples $n$ is proportional to the feature dimension $d$ with a small $d/n$ (Section~\ref{section:quantile}). More quantitatively, quantile regression at nominal level $\alpha\in(0.5, 1)$ roughly achieves coverage $\alpha-(\alpha-1/2)d/n$ regardless of the noise distribution. To the best of our knowledge, this is the first rigorous characterization of the under-coverage bias in quantile regression.

\item Towards understanding the source of this under-coverage bias, we disentangle the effect of estimating the bias and estimating the linear coefficient on the coverage of the learned linear quantile (Section~\ref{section:extension}). We show that the estimation error in the bias can have either an under-coverage or over-coverage effect, depending on the noise distribution. In contrast, the estimation error in the linear coefficient always drives the learned quantile to under-cover, and we show this effect is present even on broader classes of data distributions beyond the Gaussian linear model.
  
\item We perform experiments on simulated and real data to test our theory (Section~\ref{section:experiments}). Our simulations show that the coverage of quantile regression in Gaussian linear models agrees well with our precise theoretical formula as well as the $\alpha-(\alpha-1/2)d/n$ approximation. On real data, we find quantile regression using high-capacity models (such as neural networks) exhibits severe under-coverage biases, while linear quantile regression can also have a mild but non-negligible amount of under-coverage, even after we remove the potential effect of model misspecification.
  
\item On the technical end, our analysis builds on recent understandings of empirical risk minimization problems in the high-dimensional proportional limit with a small $d/n$, and develops new techniques such as a novel concentration argument to deal with an additional learnable variable in learning linear models with biases, which we believe could be of further interest (Section~\ref{section:proof-sketch}).
\end{itemize}

\subsection{Related work}



\paragraph{Algorithms for uncertainty estimation in regression}




The earliest methods for uncertainty estimation in regression adopted subsampling methods (bootstrap) or leave-one-out methods (Jackknife) for assessing or calibrating prediction uncertainty~\citep{quenouille1949approximate, tukey1958bias, stone1974cross, geisser1975predictive}. More recently, a growing line of work builds on the idea of conformal prediction~\citep{shafer2008tutorial} to design uncertainty estimation algorithms for regression. These algorithms provide confidence bounds or prediction intervals by post-processing any predictor, and can achieve distribution-free finite-sample marginal coverage guarantees utilizing exchangeability of the data~\citep{papadopoulos2008inductive, vovk2012conditional, lei2018distribution, romano2019conformalized, kivaranovic2020adaptive,vovk2015cross, vovk2018cross, vovk2005algorithmic, barber2021predictive}. Further modifications of the conformal prediction technique can yield stronger guarantees such as group coverage~\citep{barber2019limits} or coverage under distribution shift~\citep{barber2019conformal} under additional assumptions. Our under-coverage results advocate the necessity of such post-processing techniques, and are complementary in the sense that we provide understandings on the more vanilla quantile regression algorithm. Quantiles and prediction intervals can also be obtained by aggregating multiple predictors, such as using Bayesian neural networks~\citep{mackay1992bayesian, gal2016dropout, kendall2017uncertainties,malinin2018predictive,maddox2019simple} or ensembles~\citep{lakshminarayanan2016simple, ovadia2019can, huang2017snapshot,malinin2019ensemble}. These methods offer an alternative approach for uncertainty estimation, but do not typically come with coverage guarantees.

\paragraph{Theoretical analysis of quantile regression}
Linear quantile regression with the pinball loss dates back to the late 1970s~\citep{koenker1978regression}. The same work proved the asymptotic normality of the regression coefficients in the $n \to \infty$, fixed $d$ limit.
\citet{takeuchi2006nonparametric} studied non-parametric quantile regression using kernel methods, and provided generalization bounds (with the pinball loss) based on the Rademacher complexity. \citet{meinshausen2006quantile} studied non-parametric quantile regression using random forest and showed its consistency under proper assumptions. \citet{christmann2007svms, steinwart2011estimating} established a ``self-calibration'' inequality for the quantile loss, which, when combined with standard generalization bounds, can be translated to an estimation error bound for quantile regression. These works all focus on bounding the parameter or function estimation error, which can be translated to bounds on the coverage bias, but does not tell the sign of this coverage bias as we do in this paper. We also remark that conformalization can be used in conjunction with quantile regression to correct is coverage bias~\citep{romano2019conformalized}.



\paragraph{Uncertainty quantification for classification}
For classification problems, two main types of uncertainty quantification methods have been considered: outputting discrete prediction sets with guarantees of covering the true (discrete) label~\citep{wilks1941determination, wilks1942statistical,lei2014classification,angelopoulos2020uncertainty,bates2021distribution,cauchois2021knowing,cauchois2020robust}, or calibrating the predicted probabilities~\citep{platt1999probabilistic, zadrozny2001obtaining, zadrozny2002transforming, lakshminarayanan2016simple, guo2017calibration}. The connection between prediction sets and calibration was discussed in~\citep{gupta2020distribution}. The sample complexity of calibration has been studied in a number of theoretical works~\citep{kumar2019verified, gupta2020distribution,shabat2020sample,jung2020moment,liu2019implicit,bai2021don}. Our work is inspired by the recent work of~\citet{bai2021don}, which showed that logistic regression is over-confident even if the model is correctly specified and the sample size is larger than the dimension.

\paragraph{High-dimensional behaviors of empirical risk minimization}
There is a rapidly growing literature on limiting characterizations of convex optimization-based estimators in the $n \propto d$ regime \citep{donoho2009message, bayati2011dynamics, el2013robust, karoui2013asymptotic, stojnic2013framework, thrampoulidis2015regularized, donoho2016high, thrampoulidis2018precise, mai2019large, sur2019modern, candes2020phase}. 
Our analysis builds on results for unregularized M-estimator derived in \citep{thrampoulidis2018precise} and generalizes theirs in certain aspects (see also~\citep{el2013robust, donoho2016high, karoui2013asymptotic}).

\section{Preliminaries}
\label{section:prelim}

In this paper we focus on the problem of learning quantiles. Suppose we observe a training dataset $\set{(\xb_i, y_i)}_{i=1}^n$ drawn i.i.d. from some joint distribution $\P$ on $\R^d\times \R$, where $\xb_i\in\R^d$ is the input features and $y\in\R$ is the real-valued response (label).
Let $F(t|\xb)\defeq \P(Y\le t | \Xb=\xb)$ denote the conditional CDF of $Y|\Xb$. Our goal is to learn the $\alpha$-(conditional) quantile of $Y|\Xb$:
\begin{align*}
  q_\alpha^\star(\xb) \defeq \inf\set{t\in\R: F(t|\xb) \ge \alpha}.
\end{align*}
For example, $q^\star_{0.95}(\xb)$ is the ground truth $95\%$ quantile of the true conditional distribution $Y|\Xb$, and can be seen as the ``ideal'' $95\%$ upper confidence bound for the label $y$ given the features $\xb$. Throughout this paper we work with upper quantiles, that is, $\alpha\in(0.5, 1)$ (some typical choices are $\alpha\in\set{0.8, 0.9, 0.95}$); by symmetry our results hold for learning lower quantiles as well.

\paragraph{Coverage}
For any learned quantile function $\what{f}:\R^d\to \R$, the marginal coverage (henceforth ``coverage'') of $\what{f}$ is the probability of $y\le \what{f}(\xb)$ on a new test example $(\xb,y)$:
\begin{align}\label{eqn:coverage_def}
  \coverage(\what{f}) \defeq \P_{(\xb, y)}\paren{ y \le \what{f}(\xb) } = \E_{\xb} \brac{ \P\paren{y \le \what{f}(\xb) | \xb}}.
\end{align}
For learning the $\alpha$-quantile ($\alpha> 0.5$), we usually expect $\coverage(\what{f}) \approx \alpha$, i.e. $\what{f}(\xb)$ covers the label $y$ on approximately $\alpha$ proportion of the data, under the ground truth data distribution. 


We say that $\what{f}$ has \emph{under-coverage} if $\coverage(\what{f})<\alpha$ and \emph{over-coverage} if $\coverage(\what{f})>\alpha$. Note that these two notions are not symmetric: Over-coverage means that the learned upper quantile $\what{f}(\xb)$ is overly conservative (higher than enough), and is typically tolerable; In contrast, under-coverage means that $\what{f}(\xb)$ fails to cover $y$ with $\alpha$ probability, and is typically considered as a failure. We remark that while there exist more fine-grained notions of coverage such as conditional coverage~\citep{barber2019limits}, the (marginal) coverage is still a basic requirement for any quantile learning algorithm. 



\paragraph{Quantile regression}
We consider quantile regression, a standard method for learning quantiles from data~\citep{koenker2001quantile}. Quantile regression estimates the true quantile function $q_\alpha(\cdot)$ via the \emph{pinball loss}~\citep{koenker1978regression,steinwart2011estimating}
\begin{align}
  \label{equation:pinball-loss}
  \ell^\alpha(t) = -(1-\alpha)t\indic{t \le 0} + \alpha t\indic{t > 0}.
\end{align}
Note that in the special case of $\alpha=0.5$, we have $\ell^{0.5}(t)=|t|/2$, and thus the pinball loss strictly generalizes the absolute loss (for learning medians) to learning any quantile. Given the training dataset and any function class $\set{f_\theta: \theta\in\Theta}$ (e.g. linear models or neural networks), quantile regression solves the (unregularized) empirical risk minimization (ERM) problem 
\begin{align}\label{eqn:quantile_regression_f_theta}
  \what{\theta} = \argmin_{\theta\in\Theta} \what{R}_n(\theta) \defeq \frac{1}{n}\sum_{i=1}^n \ell^\alpha\paren{ y_i - f_\theta(\xb_i)}.
\end{align}
(We take $\what{\theta}$ as any minimizer of $\what{R}_n$ when the minimizer is non-unique.) Let $R(\theta)\defeq \E[\what{R}_n(\theta)]$ denote the corresponding population risk. It is known that the population risk over all (measurable) functions is minimized at the true quantile $q^\star_\alpha = \argmin_{f} R(f)$ under minimal regularity conditions (for completeness, we provide a proof in Appendix~\ref{appendix:qr-consistency}).

\section{Quantile regression exhibits under-coverage}
\label{section:quantile}
We analyze quantile regression in the vanilla setting where the input distribution is a standard Gaussian and $y$ follows a linear model of $\xb$:
\begin{align}
  \label{equation:linear-model}
  y = \wb_\star^\top\xb + z,~~~\textrm{where}~~~\xb\sim\normal(\bzero, \Ib_d),~~z\sim P_z.
\end{align}
Above, $\wb_\star\in\R^d$ is the ground truth coefficient vector, and the noise $z\sim P_z$ is independent of $\xb$. The Gaussian input assumption is required only for technical convenience in the high-dimensional limiting analysis, and we believe it is not strictly required for the same result to hold\footnote{Our results can be extended directly to any correlated Gaussian input $\xb\sim \normal(\bzero, \bSigma)$ by the transform $\wt{\xb}= \bSigma^{-1/2}\xb$ and $\wt{\wb}_\star=\bSigma^{1/2}\wb_\star$. We believe our results also hold for i.i.d. sub-Gaussian inputs by the universality principle (e.g.~\citep{bayati2015universality}).} (an extension to more general input distributions can also be found in Section~\ref{section:extension}). The noise distribution $P_z$ is required to satisfy the following smoothness assumption, but can otherwise be arbitrary:
\begin{assumption}[Smooth density]
  \label{ass:noise_density}
  The noise distribution $P_z$ has a smooth density $\phi_z \in C^\infty(\R)$ (with corresponding CDF $\Phi_z$), with bounded derivatives: $\sup_{t \in \R} \vert \phi_z^{(k)}(t) \vert < \infty$ for any $k\ge 0$. We further assume that $\phi_z(z_\alpha) > 0$, where $z_\alpha\defeq \inf\set{t\in\R: \Phi_z(t) \ge \alpha}$ is the $\alpha$-quantile of $P_z$.
\end{assumption}

Under the above model, it is straightforward to see that the true $\alpha$-conditional quantile of $y|\xb$ is also a linear model (with bias):
\begin{align}
  \label{equation:linear-true-quantile}
  q^\star_\alpha(\xb) = \wb_\star^\top \xb + z_\alpha. 
\end{align}
Given the training data $\set{(\xb_i, y_i)}_{i=1}^n$, we learn a linear quantile function $\hat{f}(\xb) = \hat{\wb}^\top\xb +  \hat{b}$ via quantile regression: 
\begin{align}
  \label{equation:quantile-reg-linear}
  (\hat{\wb}, \hat{b}) = \argmin_{\wb, b} \what{R}_n(\wb, b) \defeq \frac{1}{n}\sum_{i=1}^n \ell^\alpha(y_i - (\wb^\top \xb_i + b)),
\end{align}
where $\ell^\alpha$ is the pinball loss in~\eqref{equation:pinball-loss}. As our linear function class realizes the true quantile function~\eqref{equation:linear-true-quantile}, the population risk is minimized at the true quantile: $\argmin_{\wb, b} R(\wb, b) = (\wb_\star, z_\alpha)$.


We are now ready to state our main result, which shows that quantile regression exhibits an inherent under-coverage bias even in this vanilla realizable setting. 
\begin{theorem}[Quantile regression exhibits under-coverage bias]
  \label{theorem:main}
  Suppose the data is generated from the linear model~\eqref{equation:linear-model} and the noise satisfies Assumption~\ref{ass:noise_density}. Let $\hat{f}(\xb)=\hat{\wb}^\top\xb+\hat{b}$ be the output of quantile regression~\eqref{equation:quantile-reg-linear} at level $\alpha\in(0.5, 1)$. Then, in the limit of $n,d\to\infty$ and $d/n\to\kappa$ where $\kappa\in(0, \kappa_0]$ for some small $\kappa_0>0$, for the coverage (\ref{eqn:coverage_def}), we have ($\gotop$ denotes convergence in probability)
  \begin{align*}
    \coverage(\hat{f}) \gotop \alpha - C_{\alpha, \kappa}~~~\textrm{for some}~C_{\alpha, \kappa}>0.
  \end{align*}
  That is, the limiting coverage of the learned quantile function is less than $\alpha$. Further, for small enough $\kappa$ we have the local linear expansion
  \begin{align}
    \label{equation:alpha-minus-onehalf}
    C_{\alpha, \kappa} = (\alpha-1/2) \kappa + o(\kappa).
  \end{align}
\end{theorem}
Theorem~\ref{theorem:main} builds on the precise characterization of ERM problems in the high-dimensional proportional limit~\citep{thrampoulidis2018precise}, along with new techniques over existing work for dealing with the unique challenges in quantile regression (such as analyzing the additional learnable bias $b$). An overview of the main technical steps is provided in Section~\ref{section:proof-sketch}, and the full proof is deferred to Appendix~\ref{sec:proof_main}.

\paragraph{Implications}
Theorem~\ref{theorem:main} can be illustrated by the following numeric example. Suppose we perform quantile regression at $\alpha=0.9$, where the data follows the linear model~\eqref{equation:linear-model}, and our $\kappa=d/n=0.1$ (so that the sample size is 10x number of parameters). Then Theorem~\ref{theorem:main} shows that, even in this realizable, under-parametrized setting, the coverage of the learned quantile $\hat{f}$ is going to be roughly $0.9 - C_{\alpha, \kappa}$ when $n,d$ are large, and further $C_{\alpha, \kappa}\approx (\alpha-1/2)\kappa=0.04$. Thus the actual coverage is around $0.9-0.04=0.86$, and such a $4\%$ under-coverage bias can be rather non-negligible in reality. 

We remark that a symmetric conclusion of Theorem~\ref{theorem:main} also holds for lower quantiles, and we further expect similar results also hold for learning prediction intervals, where the coverage is defined as the two-sided coverage of the prediction interval formed by the learned \{lower quantile, upper quantile\}. To the best of our knowledge, this offers a first precise theoretical understanding of why practically trained quantiles or prediction intervals often under-cover than the desired coverage level~\citep{romano2019conformalized}.

\vspace{-0.8em}
\paragraph{Comparison against existing theories}
An important feature of the under-coverage bias shown in Theorem~\ref{theorem:main} is that it only shows up in the $n,d$ proportional regime, and is not implied by existing theories on quantile regression. Classical asymptotic theory only shows asymptotic normality $\sqrt{n}([\hat{\wb}, \hat{b}] - [\wb_\star, z_\alpha])\to\normal(\bzero, \Vb)$ in the $n\to\infty$, fixed $d$ limit~\citep{koenker1978regression,van2000asymptotic}. Under this limit, $\coverage(\hat{f})$ is consistent at $\alpha$ with $O(1/\sqrt{n})$ deviation. \citet{christmann2007svms,steinwart2011estimating}~consider the finite $n,d$ setting and establish \emph{self-calibration inequalities} (similar to strong convexity) that bounds the quantile estimation error by the square root excess loss $\sqrt{R(\hat{f})-R(q^\star_\alpha)}$. Combined with standard generalization theories (e.g. via Rademacher complexities) and Lipschitzness, this can be turned into a bound on $|\coverage(\hat{f}) - \alpha|$, but does not tell the sign (positive or negative) of the coverage bias.

\vspace{-0.8em}
\paragraph{Large $\kappa$; extension to over-parametrized learning}
While Theorem~\ref{theorem:main} requires a small $\kappa=d/n$, the approximation formula~\eqref{equation:alpha-minus-onehalf} suggests that the over-coverage should get more severe as $\kappa$---the measure of over-parametrization in this problem---gets larger. We confirm this trend experimentally in our simulations in Section~\ref{section:simulations}.

As an extension to Theorem~\ref{theorem:main}, we also show theoretically that the under-coverage bias indeed becomes even more severe in over-parametrized learning, under the same linear model~\eqref{equation:linear-model}: When $d/n>\wt{O}(1)$, and the noise $P_z$ is sub-Gaussian and symmetrically distributed about $0$, the convergence point of the gradient descent path on the quantile regression risk $\hat{R}_n$ is the minimum-norm interpolator of the data, which has coverage $0.5\pm \wt{O}(1/\sqrt{d})$ with high probability (see Appendix~\ref{appendix:overparam} for the formal statement and the proof). Notably, this $0.5$ coverage does not depend on $\alpha$ and exhibits a severe under-coverage.









\section{Understanding the source of the under-coverage bias}
\label{section:extension}

In this section, we take steps towards a deeper understanding of how the under-coverage bias shown in Theorem~\ref{theorem:main} happens. Recall that the quantile regression returns $\hat{f}(\xb)=\hat{\wb}^\top\xb + \hat{b}$ where $(\hat{\wb}, \hat{b})$ is a solution to the ERM problem~\eqref{equation:quantile-reg-linear} and estimates the true parameters $(\wb_\star, z_\alpha)$. Our main approach in this section is to \emph{disentangle the effect of the two sources}---the estimation error in $\hat{b}$ and the estimation error in $\wb$---on the coverage of $\hat{f}$.

We show that the estimation error in $\hat{b}$ \emph{can have either an under-coverage or an over-coverage effect}, depending on the noise distribution (Section~\ref{section:effect-b}). In contrast, the estimation error in $\hat{\wb}$ \emph{always has an under-coverage effect}; this holds not only for the linear model assumed in Theorem~\ref{theorem:main}, but also on more general data distributions (Section~\ref{section:effect-w}). In the setting of Theorem~\ref{theorem:main}, this under-coverage effect of $\hat{\wb}$ is always strong enough to dominate the effect of $\hat{b}$, leading to the overall under-coverage.




\subsection{Effect of estimation error in $\hat{b}$}
\label{section:effect-b}
To study the effect of $\hat{b}$, we use the quantity $\hat{b}-z_\alpha$ as a measure for its effect on the coverage---Recall that the true quantile is $q^\star_\alpha(\xb)=\wb_\star^\top + z_\alpha$, thus having $\hat{b}<z_\alpha$ means that $\hat{b}$ contributes to under-coverage, whereas $\hat{b}>z_\alpha$ means $\hat{b}$ contributes to over-coverage. (This can be seen more straightforwardly in the easier case where we know $\wb_\star$ and only output $\hat{b}$ to estimate $z_\alpha$.)


The following corollary shows that, under the same settings of Theorem~\ref{theorem:main}, the error $\hat{b}-z_\alpha$ can be understood precisely. The proof can be found in Appendix~\ref{appendix:proof-b}.
\begin{corollary}[Effect of $\hat{b}$ on coverage depends on noise distribution]
  \label{corollary:b}
  Under the same settings as Theorem~\ref{theorem:main}, for any $\alpha\in(0.5, 1)$, as $n,d\to\infty$ with $d/n\to\kappa\in(0, \kappa_0]$, we have
  \begin{enumerate}[label=(\alph*), leftmargin=1.5pc]
  \item The learned bias $\hat{b}$ from quantile regression~\eqref{equation:quantile-reg-linear} converges to the following limit:
    \begin{align*}
      \hat{b} -z_\alpha \gotop C_{\alpha, \kappa}^b = \barb_0\kappa + o(\kappa),
    \end{align*}
    where $\barb_0$ has a closed-form expression:
    \begin{align}
      \label{equation:barb0-maintext}
      \barb_0  \defeq \frac{-\alpha(1-\alpha)\phi_z'(z_\alpha) - (2\alpha-1)\phi_z^2(z_\alpha)}{2\phi_z^3(z_\alpha)}.
    \end{align}
  \item For any $\alpha\in(0.5, 1)$, when $P_z$ is the Gaussian distribution (with arbitrary scale), we have $\barb_0<0$ in which case $C_{\alpha, \kappa}^b<0$ for small enough $\kappa$. Conversely, for any $\alpha\in(0.5, 1)$, there exists some noise distribution $P_z$ for which $\barb_0>0$, in which case $C_{\alpha, \kappa}^b>0$ for small enough $\kappa$.
  \end{enumerate}
\end{corollary}
Corollary~\ref{corollary:b} shows that the sign of $C_{\alpha,\kappa}^b$ in the limiting regime (and thus the effect of $\hat{b}$ on the coverage) depends on $\barb_0$, which in turn depends on the noise distribution $P_z$. For common noise distributions such as Gaussian we have $C_{\alpha, \kappa}^b<0$ at small $\kappa$, but there also exists $P_z$ such that $C_{\alpha,\kappa}^b>0$. Note that the second claim in part (b) follows directly from~\eqref{equation:barb0-maintext}: we can always design the density $\phi_z$ by varying $\phi_z(z_\alpha)$ and $\phi_z'(z_\alpha)$ so that $\barb_0>0$. Overall, this result shows that the under-coverage bias in Theorem~\ref{theorem:main} cannot be simply explained by the under-estimation error in $\hat{b}$.



\subsection{Effect of estimation error in $\hat{\wb}$; relaxed data distributions}
\label{section:effect-w}



We now show that the primary source of the under-coverage is the estimation error in $\hat{\wb}$, which happens not only on the linear data distribution assumed in Theorem~\ref{theorem:main}, but also on a broader class of data distributions. We consider the following relaxed data distribution assumption
\begin{align}
  \label{equation:relaxed-model}
  y = \mu_\star(\xb) + \sigma_\star(\xb) z,
\end{align}
where the noise $z\sim P_z$. We do not put structural assumptions on $(\mu_\star,\sigma_\star)$, except that we assume the true $\alpha$-quantile is still a linear function of $\xb$, that is, there exists $(\wb_\star, b_\star)$ for which
\begin{align}
  \label{equation:linear-true-quantile-relaxed}
  q_\alpha^\star(\xb) = \mu_\star(\xb) + \sigma_\star(\xb)z_\alpha = \wb_\star^\top \xb + b_\star.
\end{align}
Since here we are interested in the effect of estimating $\wb_\star$, for simplicity, we assume that we know $b_\star$ and only estimate $\wb_\star$ via some estimator $\hat{\wb}$. We now collect our assumptions and state the result. 
\begin{assumption}[Relaxed data distribution]
  \label{assumption:relaxed-data}
  The data is distributed as model~\eqref{equation:relaxed-model} with a linear $\alpha$-quantile function~\eqref{equation:linear-true-quantile-relaxed}. Further, the data distribution satisfies the following regularity conditions:
  \begin{enumerate}[label=(\alph*), leftmargin=1.5pc]
  \item The distribution of $\xb\in\R^d$ is symmetric about $\bzero$, has a lower bounded covariance $ \E[\xb\xb^\top]\succeq \underline{\gamma}\Ib_d$, and is $K$-sub-Gaussian, for constants $\underline{\gamma}, K>0$.
  \item The variance function $\sigma_\star(\cdot)$ is bounded and symmetric: For all $\xb\in\R^d$ we have $\underline{\sigma}\le \sigma_\star(\xb)\le \overline{\sigma}$ for some constants $\underline{\sigma}, \overline{\sigma}>0$, and $\sigma_\star(\xb)=\sigma_\star(-\xb)$.
  \item The noise density $\phi_z$ is continuously differentiable and symmetric about $0$, i.e. $\phi_z(t)=\phi_z(-t)$ for all $t\in\R$. Further, $\phi_z$ is uni-modal, i.e. $\phi_z'(t)|_{t<0}>0$ and $\phi_z'(t)|_{t>0}<0$.
  \end{enumerate}
\end{assumption}
\begin{theorem}[Estimation error in $\hat{\wb}$ leads to under-coverage on a family of data distributions]
  \label{theorem:w}
  Under the relaxed data distribution assumption (Assumption~\ref{assumption:relaxed-data}), for any $\alpha>3/4$, there exists constants $c,r_0>0$ such that for any learned quantile estimate $\hat{f}(\xb)=\hat{\wb}^\top\xb + b_\star$ with small estimation error $\ltwo{\hat{\wb} - \wb_\star}\le r_0$, we have
  \begin{align*}
    \coverage(\hat{f}) \le \alpha - c\underline{\gamma}/\overline{\sigma}^2 \cdot \ltwo{\hat{\wb} - \wb_\star}^2,
  \end{align*}
  that is, the learned quantile under-covers by at least $\Omega(\ltwo{\hat{\wb} - \wb_\star}^2)$. Above, $c>0$ is an absolute constant, and $r_0>0$ depends on $(\underline{\gamma}, \underline{\sigma}, K, \Phi_z, \alpha)$ but not $(n,d)$.
\end{theorem}
\paragraph{Implications; proof intuition}
Theorem~\ref{theorem:w} shows that, for a broad class of data distributions, any estimator $\hat{\wb}$ will under-cover by at least $\Omega(\ltwo{\hat{\wb} - \wb_\star}^2)$. If particular, any estimator satisfying $\ltwo{\hat{\wb}-\wb_\star}\asymp \wt{O}(\sqrt{d/n})$ (e.g. from standard generalization theory) will under-cover by $\wt{O}(d/n)$. This confirms that the estimation error in the (bulk) regression coefficient $\hat{\wb}$ is the primary source of the under-coverage bias, under assumptions that are more general than Theorem~\ref{theorem:main} in certain aspects (such as the distribution of $\xb$ and $y|\xb$). We remark that as opposed to Theorem~\ref{theorem:main}, Therorem~\ref{theorem:w} does not give an end-to-end characterization of any specific algorithm, but assumes we have an estimator $\hat{\wb}$ with a small error.



At a high-level, Theorem~\ref{theorem:w} follows from the fact that any estimator $\hat{\wb}^\top\xb+b_\star$ must be lower than the true quantile $\wb_\star^\top\xb+b_\star$ for some $\xb$ and higher for some other $\xb$. Averaging the coverage indicator over $\xb$, such under-coverage and over-coverage cancel out on the first-order if $\xb$ has a symmetric distribution, but aggregate to yield a \emph{under-coverage effect on the second-order} as long as \emph{$\Phi_z(t)|_{t>0}$ is concave} (which holds if $P_z$ is unimodal). The proof of Theorem~\ref{theorem:w} can be found in Appendix~\ref{appendix:proof-w}.

\begin{figure*}[t]
  \centering
  \begin{minipage}{0.43\textwidth}
    \centering
    \subcaption{Coverage against $\alpha$}\label{figure:sim-left}
    \vspace{-.7em}
    \includegraphics[width=0.98\textwidth]{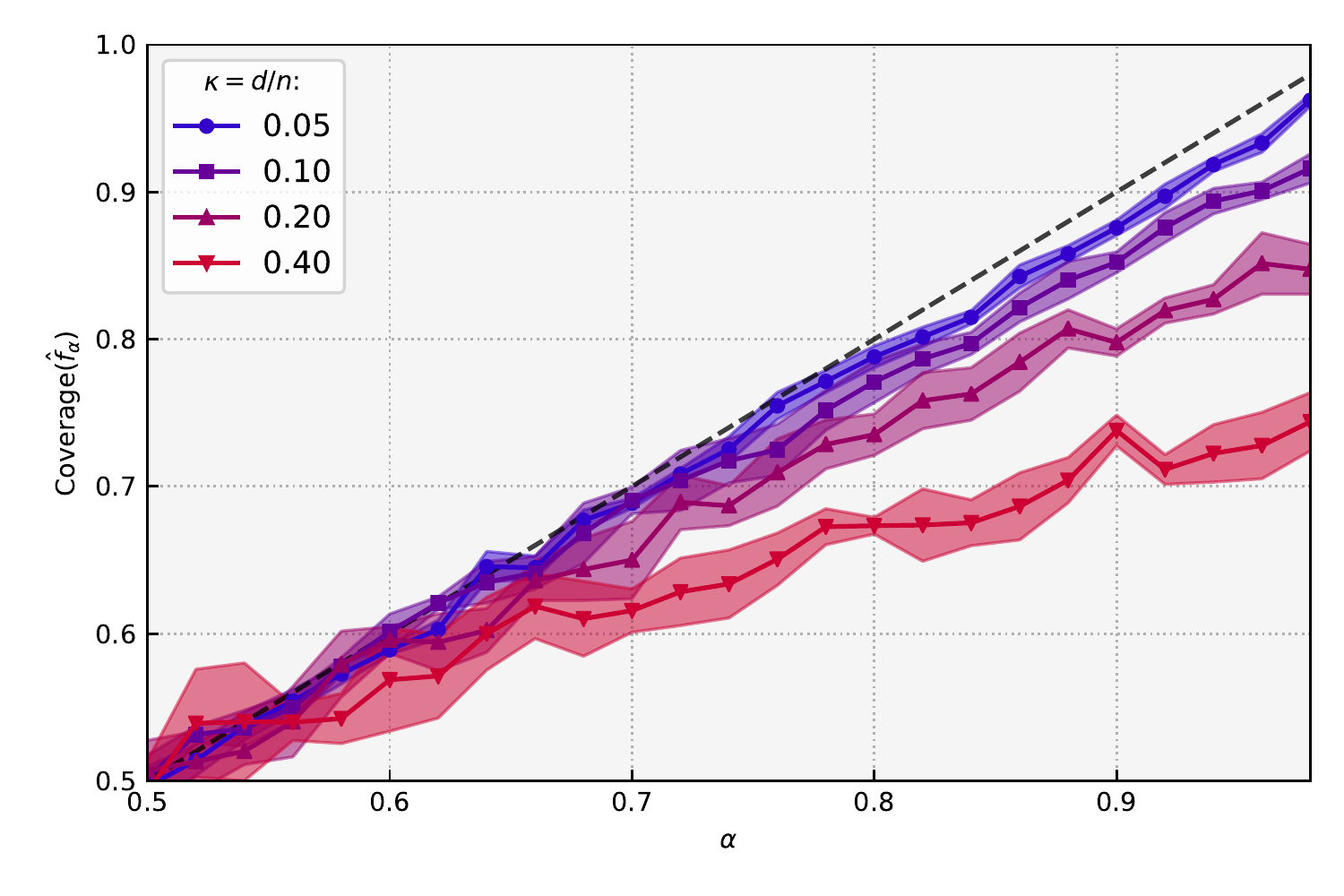}
  \end{minipage}
  \begin{minipage}{0.43\textwidth}
    \centering
    \subcaption{Coverage against $\kappa=d/n$}\label{figure:sim-right}
    \vspace{-.7em}
    \includegraphics[width=0.98\textwidth]{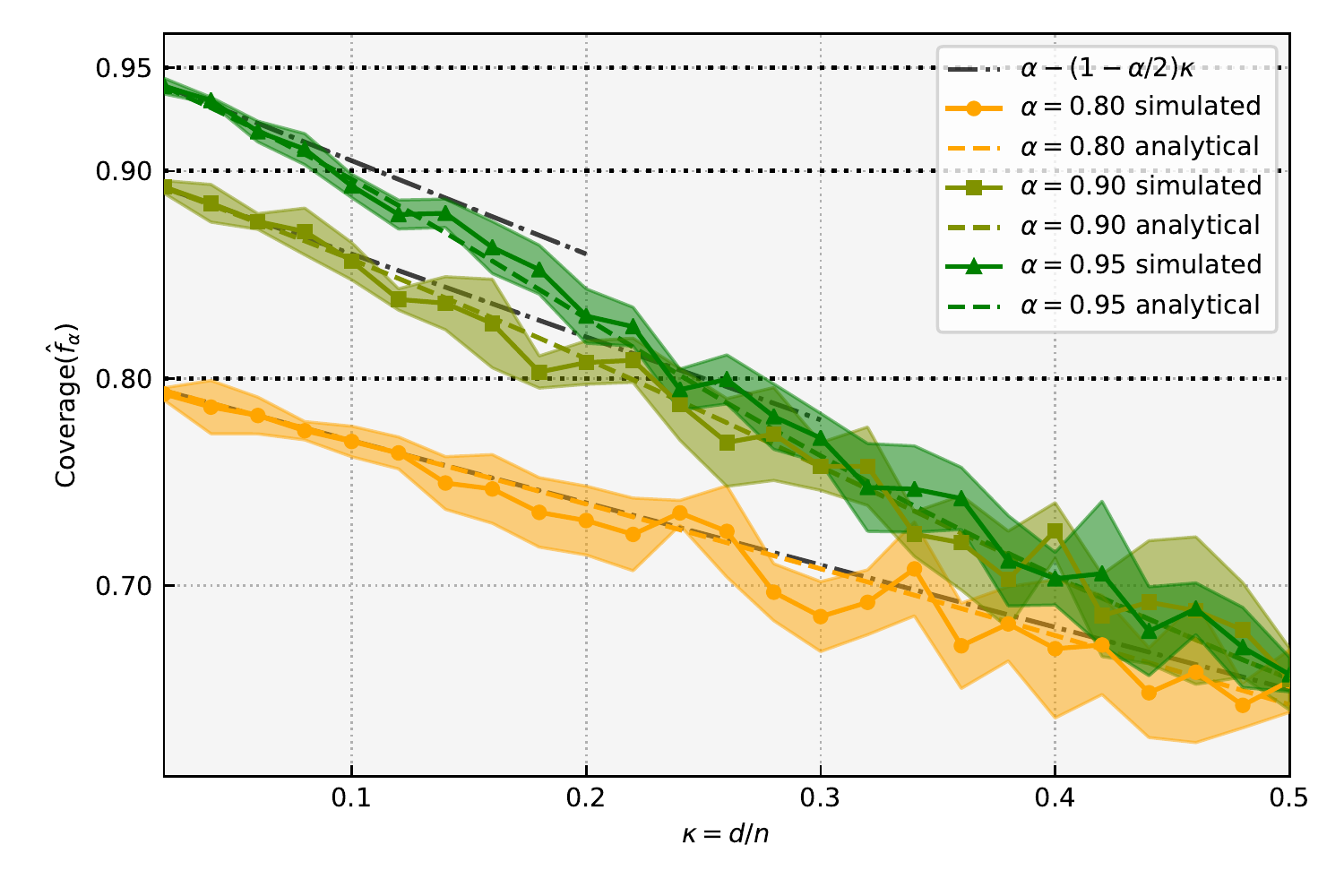}
  \end{minipage}
  \vspace{-1em}
  \caption{Coverage of quantile regression on simulated data from the realizable linear model~\eqref{equation:linear-model}. {\bf (a)(b)} Each dot represents a combination of $(\alpha, \kappa)$ and reports the mean and one-std coverage over 8 random problem instances. {\bf (a)} Coverage against the nominal quantile level $\alpha$ for fixed values of $\kappa=d/n$. {\bf (b)} Coverage against $\kappa$ for fixed $\alpha\in\set{0.8, 0.9, 0.95}$. Here ``analytical'' refers to our analyitical formula $\alpha-C_{\alpha, \kappa}$ and $\alpha-(1-\alpha/2)\kappa$ is its local linear approximation at small $\kappa$ (both from Theorem~\ref{theorem:main}).
  }
  \vspace{-1em}
  \label{figure:sim}
\end{figure*}

\section{Experiments}
\label{section:experiments}

\subsection{Simulations}
\label{section:simulations}


\paragraph{Setup}
We first test our Theorem~\ref{theorem:main} via simulations. We generate data from the linear model~\eqref{equation:linear-model} in $d=100$ dimensions with $\ltwo{\wb_\star}=1$ and noise distribution $P_z=\normal(0, 0.25)$. We vary $\kappa=d/n\in\set{0.02, 0.04, \dots, 0.5}$ where $\kappa$ determines a sample size $n$, and vary $\alpha\in\set{0.5, 0.52, ..., 0.98}$.

For each combination of $(\alpha, \kappa)$, we generate 8 random problem instances, and solve the quantile regression ERM problem~\eqref{equation:quantile-reg-linear} on each instance via (sub)-gradient descent. We evaluate the coverage of the learned quantile $\hat{f}$ (thanks to the linear model~\eqref{equation:linear-model}, the coverage can be computed exactly without needing to introduce a test set). Additional details about the setup can be found in Appendix~\ref{appendix:simulations}.

\paragraph{Results}
Figure~\ref{figure:sim} plots the coverage of the learned quantiles. Observe that quantile regression exhibits under-coverage consistently across different values of $(\alpha, \kappa)$. Figure~\ref{figure:sim-left} shows that at fixed $\kappa$, the amount of under-coverage gets more severe at a higher $\alpha$, which is qualitatively consistent with our approximation formula $(\alpha-1/2)\kappa$.
Figure~\ref{figure:sim-right} further compares simulations with our analytical formula $\alpha - C_{\alpha, \kappa}$ (found numerically through solving the system of equations~\ref{equation:system_main}), as well as the local linear approximation $\alpha-(\alpha-1/2)\kappa$ claimed in Theorem~\ref{theorem:main}. Note that the simulations agree extremely well with the analytical formula. The approximation $\alpha-(\alpha-1/2)\kappa$ is also very accurate for almost all $\kappa$ at $\alpha=0.8$, and accurate for small $\kappa$ at $\alpha=0.9,0.95$. These verify our Theorem~\ref{theorem:main} and suggests it holds at rather realistic values of the dimension ($d=100$).

\begin{table}[t]
  \small
  \caption{Coverage ($\%$) of quantile regression on real data at nominal level $\alpha=0.9$. Each entry reports the test-set coverage with mean and std over 8 random seeds. $(d,n)$ denotes the \{feature dim, \# training examples\}.}
  \label{table:real}
  \centerline{
    \begin{tabular}{l|llll|rr}
      \toprule
      Dataset & Linear & MLP-3-64 & MLP-3-512 & MLP-freeze-3-512 & $d$ & $n$ \\
      \midrule
      Community & 88.63$\pm$1.53 & 76.46$\pm$1.41 & 63.09$\pm$2.91 & 87.85$\pm$1.30 & 100 & 1599 \\
      Bike & 89.64$\pm$0.44 & 88.75$\pm$0.91 & 87.67$\pm$0.49 & 89.27$\pm$0.57 & 18 & 8708 \\
      Star & 89.48$\pm$2.56 & 83.14$\pm$1.76 & 69.71$\pm$1.82 & 88.05$\pm$2.42 & 39 & 1728 \\
      MEPS\_19 & 90.09$\pm$0.72 & 85.46$\pm$0.96 & 78.55$\pm$0.93 & 89.03$\pm$0.51 & 139 & 12628 \\
      MEPS\_20 & 90.06$\pm$0.57 & 86.52$\pm$0.65 & 80.77$\pm$0.72 & 89.60$\pm$0.28 & 139 & 14032 \\
      MEPS\_21 & 89.99$\pm$0.39 & 83.79$\pm$0.52 & 73.09$\pm$0.82 & 89.15$\pm$0.36 & 139 & 12524 \\
      \midrule
      Nominal ($\alpha$) & 90.00 & 90.00 & 90.00 & 90.00 & - & - \\
      \bottomrule
    \end{tabular}
  }
  \vspace{-1em}
\end{table}

\subsection{Real data experiments}
\label{section:real}






\paragraph{Datasets and models}
We take six real-world regression datasets: community and crimes ({\tt Community})~\citep{community}, bike sharing ({\tt Bike})~\citep{bike}, Tennessee's student teacher achievement ratio ({\tt STAR})~\citep{star}, as well as the medical expenditure survey number 19 ({\tt MEPS\_19})~\citep{meps19}, number 20 ({\tt MEPS\_20})~\citep{meps20}, and number 21 ({\tt MEPS\_21})~\citep{meps21}. All datasets are pre-processed to have standarized features and randomly split into a 80\% train set and 20\% test set.

To go beyond linear quantile functions, we perform quantile regression with one of the following four models as our $f_\theta$: linear model ({\tt Linear}), a 3-layer MLP (two non-linear layers) with width 64 ({\tt MLP-3-64}), 512 ({\tt MLP-3-512}), and a variant of the width-512 MLP where all representation layers are frozen and only the last linear layer is trained ({\tt MLP-freeze-3-512}). All linear layers include a trainable bias. We minimize the $\alpha$-quantile loss~\eqref{eqn:quantile_regression_f_theta} via momentum SGD with batch size 64. For each setting, we average over 8 random seeds where each seed determines the train-validation split, model initialization, and SGD batching. In our real experiments we fix $\alpha=0.9$. (Results at $\alpha\in\set{0.8, 0.95}$ as well as additional experimental setups can be found in Appendix~\ref{appendix:real}).

\paragraph{Results}
Table~\ref{table:real} reports the coverage of the learned quantile functions (evaluated on the test sets). Observe that all MLPs exhibit under-coverage compared with the nominal level $90\%$. Additionally, the amount of under-coverage correlates well with model capacity---the two vanilla MLPs under-covers more severely than the MLP-freeze and the linear model. Notice that the linear model does not have a notable under-coverage on most datasets---we believe this is a consequence of $d/n$ being small on these datasets. The only exception is the {\tt Community} dataset with the highest $d/n\approx 1/16$, on which the linear model does under-cover mildly by roughly $1\%$.



\subsection{Linear quantile regression on pseudo-labels}
\label{section:pseudo}
To further test the coverage of linear quantile regression on real data distributions, we make two modifications: (1) We subset the training data by fixing $d$ and reducing $n$, so as to test the coverage across differerent values of $\kappa=d/n$; (2) We compare linear quantile regression on both true labels $y_i$, and \emph{pseudo-labels} $y_i^{\rm pseudo}$ generated from estimated linear models. These pseudo-labels are generated by first fitting a linear model $\what{\wb}\in\R^d$ (with square loss) on the training data, and then generating a new label using the fitted linear model $\what{\wb}$:
\begin{align*}
  y^{\rm pseudo}_i = \what{\wb}^\top \xb_i + \what{\sigma} z_i,
\end{align*}
where $\what{\sigma}$ is estimated as $\sqrt{\what{\E}_{(\xb, y)}[(y - \what{\wb}^\top \xb)^2]}$ on a separate hold-out split, and $z_i\sim\normal(0,1)$. The motivation for the pseudo-labels is to make sure that the data comes from a true linear model, removing the potential effect of model misspecification.

Table~\ref{table:real-pseudo} shows that on the {\tt MEPS\_20} dataset, linear quantile regression exhibits under-coverage at relatively large values of $\kappa$ ($0.1, 0.2, 0.5$) for both kinds of labels. Also, there is no notable difference between pseudo-labels and true labels. This provides evidence that our theory on \emph{linear} quantile regression may hold broadly on real-world data distributions.

\begin{table}[h]
  \small
  \caption{Coverage of linear quantile regression on true labels vs. pseudo-labels.}
  \label{table:real-pseudo}
  \centerline{
    \begin{tabular}{l|llllll}
      \toprule
      $\kappa=d/n$ & 0.01 & 0.02 & 0.05 & 0.1 & 0.2 & 0.5 \\
      \midrule
      MEPS\_20 & 89.83$\pm$0.67 & 89.89$\pm$0.81 & 89.54$\pm$0.82& 88.74$\pm$1.51 & 87.15$\pm$1.52 & 84.75$\pm$1.81  \\
      MEPS\_20 Pseudo & 90.05$\pm$0.85  & 89.95$\pm$0.64  & 89.49$\pm$0.64 &88.90$\pm$1.60 & 86.96$\pm$1.30  & 83.70$\pm$2.98 \\
      \midrule
      Nominal ($\alpha$) & 90.00 & 90.00 & 90.00 & 90.00 & 90.00 & 90.00 \\
      \bottomrule
    \end{tabular}
  }
  \vspace{-1em}
\end{table}

\section{Proof overview of Theorem~\ref{theorem:main}}
\label{section:proof-sketch}


\paragraph{Closed-form expression for coverage}
Our first step is to obtain a closed-form expression for the coverage. Recall that
\[
\coverage(\what{f}) \defeq \P_{(\xb, y)}( y \le \what{f}(\xb)) = \P_{(\xb, z)}( \< \wb_\star, \xb\> + z \le \<\hat \wb, \xb \> + \hat b). 
\]
As $\xb$ is standard Gaussian, and the random variable $z$ has cumulative distribution function $\Phi_z$, standard calculation then yields the closed form expression~(Lemma~\ref{lemma:coverage-expression})
\[
\coverage(\what{f}) = \E_{G \sim \normal(0, 1)}[\Phi_z(\| \hat \wb - \wb_\star \|_2 G + \hat b)]. 
\]

\paragraph{Concentration of $\| \hat{\wb} - \wb_\star \|_2$ and $\hat{b}$}
We generalize results from recent advances in high-dimensional M-estimator in linear models~\citep{el2013robust, donoho2016high, karoui2013asymptotic, thrampoulidis2018precise} to show that $\| \hat{\wb} - \wb_\star \|_2$ and $\hat b$ obtained by quantile regression~\ref{equation:quantile-reg-linear} concentrates around fixed values in the high-dimensional limit. We show that, in the limit of $d,n\to\infty$ and $d/n\to\kappa$, the following concentration happens: 
\begin{equation}
  \label{equation:concentration-rc}
  \begin{aligned}
    \| \hat \wb - \wb_\star \|_2 \gotop \tau_\star(\kappa), ~~~{\rm and}~~~\hat b \gotop b_\star(\kappa).
  \end{aligned}
\end{equation}
Above, $\tau_\star$ and $b_\star$ are determined by the solutions of a system of nonlinear equations with three variables $(\tau, \lambda, b)$: 
\begin{equation}
  \label{equation:system_main}
  \left\{
    \begin{aligned}
      & \tau^2 \kappa = \lambda^2 \cdot \E_{(G, Z) \sim \normal(0, 1) \times P_z}[ e_{\ell_b^\alpha}'(\tau G + Z; \lambda)^2 ], \\
      & \tau \kappa = \lambda \cdot \E_{(G, Z) \sim \normal(0, 1) \times P_z}[e_{\ell_b^\alpha}'(\tau G + Z; \lambda) G], \\
      & 0 = \E_{(G, Z) \sim \normal(0, 1) \times P_z}[ e_{\ell_b^\alpha}'(\tau G + Z; \lambda) ], 
    \end{aligned}
  \right.
\end{equation}
where $e_\ell(x;\tau) = \min_{v} \frac{1}{2 \tau} ( x - v)^2 + \ell(v)$ and $\ell_b^\alpha = \ell^\alpha(t - b)$ is the shifted pinball loss (\ref{equation:pinball-loss}). 
(See Theorem~\ref{thm:ERM_limit} for the formal statement.)
This is established via two main steps: We first build on the results of~\citet{thrampoulidis2018precise} to show that a variant of the risk minimization problem with a fixed bias $b$ concentrates around the solution to the first two equations in~\eqref{equation:system_main}. We then develop a novel concentration argument to deal with the additional learnable bias $b$ in the minimization problem (\ref{equation:quantile-reg-linear}), which introduces the third equation in (\ref{equation:system_main}) that will be used in characterizing the limiting value of the minimizer $\hat b$.


The concentration~\eqref{equation:concentration-rc} implies that $\coverage(\what{f})$ also converges to the following limiting coverage value (Lemma~\ref{lem:coverage_asymptotic}):
\begin{equation}\label{eqn:sketch_coverage}
  \coverage(\hat{f}) \gotop \E_{G\sim\normal(0, 1)}\brac{ \Phi_z\paren{\tau_\star(\kappa) G + b_\star(\kappa)}} \eqdef \alpha - C_{\alpha, \kappa}.
\end{equation}

\paragraph{Calculating the limiting coverage via local linear analysis}
In this final step, as another technical crux of the proof, we further evaluate the small $\kappa$ approximation of coverage value~\eqref{eqn:sketch_coverage}, and determine the sign of $C_{\alpha,\kappa}$. This is achieved by a \emph{local linear analysis} on the solutions of the aforementioned system of equations at small $\kappa$ (Lemma~\ref{lem:local-linear-expansion})  in a similar fashion as in~\cite{bai2021don}, and a precise analysis on the interplay between the concentration values $\tau_\star$, $b_\star$, and the noise density $\phi_z$. Combining these calculations yields that $C_{\alpha, \kappa}/\kappa = - (\alpha - 1/2) + o(1)$ for small enough $\kappa$ (Lemma~\ref{lem:coverage_expansion}). As $\alpha>1/2$, this establishes Theorem~\ref{theorem:main}. All details on these analyses can be found in our proofs in  Appendix~\ref{sec:proof_main}.


%
%
\section{Conclusion}
This paper presents a first theoretical justification of the under-coverage bias in quantile regression. We prove that quantile regression suffers from an inherent under-coverage bias even in well-specified linear settings, and provide a precise quantitative characterization of the amount of the under-coverage bias on Gaussian linear models. Our theory further identifies the high-dimensional estimation error in the regression coefficient as the main source of this under-coverage bias, which holds more generally on a broad class of data distributions. We believe our work opens up several interesting directions for future work, such as analyzing non-linear quantile regression, as well as analyzing other notions of uncertainty in regression problems.

\bibliographystyle{abbrvnat}
\bibliography{bib}

\makeatletter
\def\renewtheorem#1{%
  \expandafter\let\csname#1\endcsname\relax
  \expandafter\let\csname c@#1\endcsname\relax
  \gdef\renewtheorem@envname{#1}
  \renewtheorem@secpar
}
\def\renewtheorem@secpar{\@ifnextchar[{\renewtheorem@numberedlike}{\renewtheorem@nonumberedlike}}
\def\renewtheorem@numberedlike[#1]#2{\newtheorem{\renewtheorem@envname}[#1]{#2}}
\def\renewtheorem@nonumberedlike#1{  
\def\renewtheorem@caption{#1}
\edef\renewtheorem@nowithin{\noexpand\newtheorem{\renewtheorem@envname}{\renewtheorem@caption}}
\renewtheorem@thirdpar
}
\def\renewtheorem@thirdpar{\@ifnextchar[{\renewtheorem@within}{\renewtheorem@nowithin}}
\def\renewtheorem@within[#1]{\renewtheorem@nowithin[#1]}
\makeatother

\renewtheorem{theorem}{Theorem}[section]
\renewtheorem{lemma}{Lemma}[section]
\renewtheorem{remark}{Remark}
\renewtheorem{corollary}{Corollary}[section]
\renewtheorem{observation}{Observation}[section]
\renewtheorem{proposition}{Proposition}[section]
\renewtheorem{definition}{Definition}[section]
\renewtheorem{claim}{Claim}[section]
\renewtheorem{fact}{Fact}[section]
\renewtheorem{assumption}{Assumption}[section]
\renewcommand{\theassumption}{\Alph{assumption}}
\renewtheorem{conjecture}{Conjecture}[section]

\appendix
\section{Technical tools}

\subsection{The pinball loss}

Recall that we took $\ell^\alpha: \R \to \R_{\ge 0}$ to be the pinball loss for the $\alpha$-quantile, i.e., 
\[
\ell^\alpha(t) = -(1-\alpha)t\indic{t \le 0} + \alpha t\indic{t > 0}. 
\]
We denote $\ell_b^\alpha(t) = \ell^\alpha(t - b)$ to be the shifted pinball loss. We will suppress the superscript in $\ell_b = \ell_b^\alpha$ whenever it is clear in the context. The loss function $\ell_b$ is weakly differentiable, with a weak derivative $\ell_b'$ given by
\[
\ell_b'(t) = -(1-\alpha)\indic{t \le 0} + \alpha \indic{t > 0}. 
\]

\subsection{Calculus of the Moreau envelope and prox operator}
\label{sec:Moreau_calculus}

Given a convex loss function $\ell: \R \to \R$, we define its the Moreau envelope $e_\ell: \R \times \R_{> 0} \to \R$ by
\[
e_\ell(x;\lambda) = \min_{v} \Big[ \frac{1}{2 \lambda} ( x - v )^2 + \ell(v) \Big],
\]
and the proximal operator $\prox_\ell(x; \lambda): \R \times \R_{>0}\to \R$ by 
\[
\prox_\ell(x;\lambda) = \arg \min_{v} \Big[ \frac{1}{2 \lambda} ( x - v )^2 + \ell(v) \Big]. 
\]
Since $\ell$ is convex, $\prox_\ell(x; \lambda)$ is well-defined. For $\ell = \ell_b$, we have 
\[
\begin{aligned}
\prox_{\ell_b}(x; \lambda) =&~ b \cdot \ones\{ x \in [b - (1 - \alpha) \lambda, b + \alpha \lambda] \} \\
&~+ (x - \alpha \lambda) \ones\{ x > b + \alpha \lambda \} + (x + (1 - \alpha) \lambda) \ones\{ x < b - (1 - \alpha) \lambda \}. 
\end{aligned}
\]

The function $e_{\ell_b}$ is differentiable with respect to $(x, \lambda, b)$, with derivatives
\begin{equation}\label{eqn:derivative_e_function}
\begin{aligned}
\partial_x e_{\ell_b}(x; \lambda) =&~ \frac{x - \prox_{\ell_b}(x; \lambda)}{\lambda}, \\
\partial_\lambda e_{\ell_b}(x; \lambda) =&~ -\frac{[x - \prox_{\ell_b}(x; \lambda)]^2}{2 \lambda^2} = - \frac{1}{2} (\partial_x e_{\ell_b(x; \lambda)})^2, \\
\partial_b e_{\ell_b}(x; \lambda) =&~ - \partial_x e_{\ell_b}(x; \lambda). \\
\end{aligned}
\end{equation}
The functions $\partial_x e_{\ell_b}$, $\partial_\lambda e_{\ell_b}$ and $\partial_b e_{\ell_b}$ are weakly-differentiable with respect to $(x, \lambda, b)$, with the following formulas giving one (choice of) weak derivative:
\begin{equation}\label{eqn:second_derivative_e_function}
\begin{aligned}
\partial_x \partial_x e_{\ell_b}(x; \lambda) =&~ \frac{1}{\lambda} \ones \{ \prox_{\ell_b}(x; \lambda) = b \} \ge 0, \\
\partial_\lambda \partial_x e_{\ell_b}(x; \lambda) =&~ -\frac{[x - \prox_{\ell_b}(x; \lambda)]^2}{2 \lambda^2} = - \partial_x e_{\ell_b}(x; \lambda) \partial_x \partial_x e_{\ell_b}(x; \lambda) , \\
\partial_b\partial_x e_{\ell_b}(x; \lambda) =&~ - \partial_x \partial_x e_{\ell_b}(x; \lambda), \\
\partial_\lambda \partial_b e_{\ell_b}(x; \lambda) = &~ -\partial_x e_{\ell_b}(x; \lambda) \partial_b \partial_x e_{\ell_b}(x; \lambda) = \partial_x e_{\ell_b}(x; \lambda) \partial_x \partial_x e_{\ell_b}(x; \lambda) , \\
\partial_b \partial_b e_{\ell_b}(x; \lambda) =&~ \partial_x \partial_x e_{\ell_b}(x; \lambda), \\
\partial_\lambda \partial_\lambda e_{\ell_b}(x; \lambda) =&~ - \partial_x e_{\ell_b}(x; \lambda) \partial_\lambda \partial_x e_{\ell_b}(x; \lambda) = \partial_x e_{\ell_b}(x; \lambda)^2 \partial_x \partial_x e_{\ell_b}(x; \lambda). 
\end{aligned}
\end{equation}

\subsection{Implicit function theorem}

We state the standard implicit function theorem in the following.
\begin{lemma}[Implicit function theorem]\label{lem:implicit_function}
Let $\bF(\bp, \kappa):\R^{s}\times \R_{\ge 0}\to \R^s$ be a continuously differentiable vector-valued function on $\ball(\bp_0, \eps) \times [0, \bar \kappa_0)$ for some $\bar \kappa_0 > 0$. Suppose $\bF(\bp_0, 0) = 0$ and 
\[
\sigma_{\min}(\nabla_{\bp} F(\bp_0, 0)) > 0. 
\]
Then there exists a constant $\kappa_0 > 0$ and a continuous differentiable path $\bp_\star(\kappa) \in \ball(\bp_0, \eps)$, such that 
\[
\bF(\bp_\star(\kappa), \kappa) = 0, ~~~ \forall \kappa \in [0, \kappa_0).
\]
\end{lemma}

\subsection{Other technical lemmas}

\begin{lemma}
  \label{lemma:woodbury}
  For any vectors $\ub, \vb\in\R^d$ and any positive definite matrix $\Ab\in\R^{d\times d}$, $\Ab \succ \bzero$, we have
  \begin{align*}
    \abs{\ub^\top (\Ab + \vb\vb^\top)^{-1}\vb} \le \abs{\ub^\top \Ab^{-1}\vb}.
  \end{align*}
\end{lemma}
\begin{proof}
  Recall the Sherman-Morrison-Woodbury identity for matrix inversion:
  \begin{align*}
    (\Ab + \vb\vb^\top)^{-1} = \Ab^{-1} - \frac{\Ab^{-1}\vb\vb^\top\Ab^{-1}}{1 + \vb^\top\Ab^{-1}\vb}.
  \end{align*}
  Applying this, we have
  \begin{align*}
    &\quad \abs{\ub^\top (\Ab + \vb\vb^\top)^{-1} \vb } = \abs{ \ub^\top\Ab^{-1}\vb - \ub^\top \frac{\Ab^{-1}\vb\vb^\top\Ab^{-1}}{1 + \vb^\top\Ab^{-1}\vb} \vb } \\
    & = \abs{ \ub^\top\Ab^{-1}\vb - \paren{\ub^\top \Ab^{-1}\vb} \cdot \frac{\vb^\top\Ab^{-1}\vb}{1 + \vb^\top\Ab^{-1}\vb} } \\
    & = \abs{ \paren{\ub^\top \Ab^{-1}\vb} \cdot \frac{1}{1 + \vb^\top\Ab^{-1}\vb} }
      \le \abs{ \ub^\top \Ab^{-1}\vb }.
  \end{align*}
  Above, the last line used $\vb^\top\Ab^{-1}\vb\ge 0$ since $\Ab^{-1}\succeq \bzero$. This proves the lemma.
\end{proof}

\begin{lemma}\label{lem:PD_condition}
Let $\bX \in \R^s$ be a random variable with distribution $\mu$, and let $\bu: \R^s \to \R^k$ be a continuous function. Assume that there exist $(\bx_t)_{t \in [k]}$ that are in the support of the distribution of $\bX$ (i.e., for any $t \in [k]$, we have $\mu(\set{ \bx: \| \bx_t - \bx \|_2 \le \eps}) > 0$ for any $\eps > 0$), such that $[\bu(\bx_1), \ldots, \bu(\bx_k)] \in \R^{k \times k}$ is full rank. Then we have 
\[
\E[\bu(\bX) \bu(\bX)^\top] \succ 0. 
\]
\end{lemma}

\begin{proof}[Proof of Lemma \ref{lem:PD_condition}]

We denote 
\[
\omega(\eps) = \sup_{t \in [k]} \Big[ 2 \sup_{\bx \in \ball(\bx_t, \eps)} \|\bu(\bx) - \bu(\bx_t) \|_2 \cdot \sup_{\bx \in \ball(\bx_t, \eps)} \|\bu(\bx) \|_2 + \sup_{\bx \in \ball(\bx_t, \eps)} \|\bu(\bx) - \bu(\bx_t) \|_2^2 \Big]. 
\]
Since $\bu$ is a continuous function on $\R^s$, we have 
\[
\lim_{\eps \to 0} \omega(\eps) = 0.
\]
We further denote 
\[
\nu(\eps) = \min_{t \in [k]}  \mu(\ball(\bx_t, \eps)). 
\]
Then by the fact that $(\bx_t)_{t \in [k]} \subseteq {\rm supp}(\mu)$, we have $\nu(\eps) > 0$ for any $\eps > 0$. 

Then, for any $\eps > 0$, we have 
\[
\begin{aligned}
\E[\bu(\bX) \bu(\bX)^\top] \succeq&~ \sum_{t = 1}^k \int_{\ball(\bx_t, \eps)} \bu(\bx) \bu(\bx)^\top \mu(\de \bx) \\
\succeq&~ \sum_{t = 1}^k (\bu(\bx_t) \bu(\bx_t)^\top - \omega(\eps) I_k) \nu(\eps) \\
=&~  \nu(\eps) \sum_{t = 1}^k \bu(\bx_t) \bu(\bx_t)^\top   - \omega(\eps) k \nu(\eps) I_k \\
\succeq &~ \nu(\eps)\Big[ \lambda_{\min}\Big(\sum_{t = 1}^k \bu(\bx_t) \bu(\bx_t)^\top \Big)   - \omega(\eps) k \Big] I_k. 
\end{aligned}
\]
Since $[\bu(\bx_1), \ldots, \bu(x_k)]$ has full rank, we have $\lambda_{\min}(\sum_{t = 1}^k \bu(\bx_t) \bu(\bx_t)^\top) > 0$. We can choose $\eps$ sufficiently small, so that $\lambda_{\min}(\sum_{t = 1}^k \bu(\bx_t) \bu(\bx_t)^\top)   - \omega(\eps) k > 0$. This gives $\E[\bu(\bX) \bu(\bX)^\top] \succ 0$. This proves the lemma.  
\end{proof}

\section{Properties of quantile regression}
\subsection{Population minimizer of quantile risk}
\label{appendix:qr-consistency}

We can express the population quantile risk as
\begin{align*}
  R(f) = \E\brac{\ell^\alpha(y - f(\xb))} = \E_{\xb} \E\brac{\ell^\alpha(y - f(\xb)) | \xb }.
\end{align*}
Therefore, any function $f(\xb)$ that minimizes the conditional expectation $\E[\ell^\alpha(y - f(\xb)) | \xb]$ at every $\xb$ minimizes the above risk. It is a classical result that for any distribution $P$ on $\R$, a minimizer of $\E_{y\sim P}[\ell^\alpha(y - f)]$ is the $\alpha$-quantile $q_\alpha=\inf\set{t\in\R: F(t)\ge \alpha}$, where $F$ is the CDF of $P$~\citep[Section 3]{koenker1978regression}. Therefore, the conditional quantile function $q^\star(\xb)=\argmin_f \E[\ell^\alpha(y - f(\xb)) | \xb]$ is a minimizer of the aforementioned conditional expectation at every $\xb$. This proves the claim.
\qed

\subsection{Explicit expression of coverage}
\begin{lemma}
  \label{lemma:coverage-expression}
  Under the linear model~\eqref{equation:linear-model}, for any linear quantile function $\hat{f}(\xb)=\hat{\wb}^\top\xb+\hat{b}$, the coverage of $\hat{f}$ can be expressed as
  \begin{align*}
    \coverage(\hat{f}) = \P_{(\xb, y)}\paren{y \le \hat{\wb}^\top\xb + \hat{b}} = \E_{G\sim\normal(0, 1)}\brac{ \Phi_z\paren{\ltwo{\hat{\wb} - \wb_\star}G + \hat{b}}}.
  \end{align*}
\end{lemma}
\begin{proof}
  By the linear model~\eqref{equation:linear-model}, we have $y=\wb_\star^\top\xb+z$ and thus
  \begin{align*}
    & \quad \P_{(\xb, y)}\paren{y \le \hat{\wb}^\top\xb + \hat{b}} = \P_{(\xb, z)}\paren{\wb_\star^\top \xb + z \le \hat{\wb}^\top \xb + \hat{b}} \\
    & = \P_{(\xb, z)}\paren{z \le (\hat{\wb} - \wb_\star)^\top\xb + \hat{b}} \\
    & = \E_{\xb}\brac{\Phi_z\paren{(\hat{\wb} - \wb_\star)^\top\xb + \hat{b}}  } \\
    & = \E_{G\sim \normal(0,1)}\brac{\Phi_z\paren{\ltwo{\hat{\wb} - \wb_\star}G + \hat{b}} }.
  \end{align*}
  Above, the last step used the Gaussian input assumption $\xb\sim\normal(\bzero, \Ib_d)$.
\end{proof}
\section{Proof of Theorem \ref{theorem:main}}
\label{sec:proof_main}

Recall that $\ell_b^\alpha(t) = \ell^\alpha(t-b)$ where $\ell^\alpha(t)$ is the pinball loss for the $\alpha$-quantile, i.e., 
\[
\ell^\alpha(t) = -(1-\alpha)t\indic{t \le 0} + \alpha t\indic{t > 0}. 
\]
We will consider a fixed $\alpha$, so we often write $\ell_b \equiv \ell_b^\alpha$. We further define 
\[
e_\ell(x;\lambda) \defeq \min_{v\in\R} \brac{\frac{1}{2 \lambda} (x - v)^2 + \ell(v)}. 
\]
We consider the following system of equations in three variables $(\tau, \lambda, b)\in\R_{>0}\times \R_{>0}\times \R$, which will be key to our analysis of the quantile ERM problem~\eqref{equation:quantile-reg-linear}:
\begin{equation}
  \label{equation:system}
  \left\{
    \begin{aligned}
      & \tau^2 \kappa = \lambda^2 \cdot \E\brac{ e_{\ell_b}'(\tau G + Z; \lambda)^2 }, \\
      & \tau \kappa = \lambda \cdot \E\brac{e_{\ell_b}'(\tau G + Z; \lambda) G}, \\
      & 0 = \E\brac{ e_{\ell_b}'(\tau G + Z; \lambda) }.
    \end{aligned}
  \right.
\end{equation}

The following two lemmas show that the system of equations~\eqref{equation:system} has a unique solution, which further admits a local linear expansion over $\kappa$ with closed-form coefficients.
\begin{lemma}[Existence of unique solution]
  \label{lem:scf_exist_unique}
  There exists $\kappa_0>0$ such that for any $\kappa \in (0, \kappa_0]$, there exists a unique solution $(\tau_\star(\kappa), \lambda_\star(\kappa), b_\star(\kappa))$ of the system of equations~\eqref{equation:system}.
\end{lemma}


Define constants
\begin{equation}\label{eqn:def_p0}
\begin{aligned}
  & \bartau_0^2 \defeq \frac{\alpha(1-\alpha)}{\phi_z^2(z_\alpha)}, \\
  & \barlambda_0 \defeq \frac{1}{\phi_z(z_\alpha)}, \\
  & \barb_0 \defeq \frac{-\alpha(1-\alpha)\phi_z'(z_\alpha)-(2\alpha-1)\phi_z^2(z_\alpha)}{2\phi_z^3(z_\alpha)}.
\end{aligned}
\end{equation}

\begin{lemma}[Local linear expansion of solution at small $\kappa$]
\label{lem:local-linear-expansion}
Let $(\tau_\star(\kappa), \lambda_\star(\kappa), b_\star(\kappa))$ denote the solutions to~\eqref{equation:system} for any $\kappa\in(0, \kappa_0]$. The following local linear expansion holds at small $\kappa$:
\begin{equation}\label{eqn:small_kappa_expansion}
\begin{aligned}
& \tau_\star^2(\kappa) = \bartau_0^2 \kappa + o(\kappa), \\
& \lambda_\star(\kappa) = \barlambda_0 \kappa + o(\kappa), \\
& b_\star(\kappa) = z_\alpha + \barb_0 \kappa + o(\kappa), 
\end{aligned}
\end{equation}
where $z_\alpha = \Phi_z^{-1}(\alpha)$ is the $\alpha$-quantile of $P_z$.
\end{lemma}

We now show that the quantile ERM problem~\eqref{equation:quantile-reg-linear} exhibits a sharp concentration in the proportional limit ($n,d\to\infty$, $d/n\to\kappa$) where the concentration values are determined by the solutions $(\tau_\star^2(\kappa), \lambda_\star(\kappa), b_\star(\kappa))$ above. This result is a novel extension of (the unregularized case of)~\citep[Theorem 4.1]{thrampoulidis2018precise} in that it incorporates---and proves the concentration in presence of---the additional trainable bias parameter $b$.
Recall the ERM problem~\eqref{equation:quantile-reg-linear} is
\begin{align}
(\hat{\wb}, \hat{b}) \in \argmin_{\wb, b} \what{R}_n(\wb, b) \defeq \frac{1}{n}\sum_{i=1}^n \ell^\alpha(y_i - (\wb^\top \xb_i + b)). 
\end{align}

\begin{theorem}[Concentration of quantile ERM]
  \label{thm:ERM_limit}
  Under the linear model~\eqref{equation:linear-model} and Assumption \ref{ass:noise_density}, consider the limit $n, d \to \infty$ and $d / n \to \kappa\in(0,\kappa_0]$ where $\kappa_0>0$ is some constant. Then with probability approaching one, the empirical risk minimizer $(\hat \wb, \hat b)$ exists (but may not be unique), and for any empirical risk minimizer $(\hat \wb, \hat b)$, we have 
\[
\hat b \stackrel{p}{\rightarrow} b_\star(\kappa), ~~~~ \ltwo{\hat{\wb} - \wb_\star}^2 \stackrel{p}{\rightarrow} \tau_\star^2(\kappa). 
\]
\end{theorem}

Denote 
\[
\coverage_{\alpha, \kappa} \equiv \E_{G\sim\normal(0, 1)}\brac{ \Phi_z\paren{\tau_\star(\kappa) G + b_\star(\kappa)}}. 
\]
Combining Theorem \ref{thm:ERM_limit}, Lemma \ref{lem:local-linear-expansion}, and the expression of the coverage in Lemma~\ref{lemma:coverage-expression}, the following two lemmas show that $\coverage(\hat{f})$ also concentrates around a value $\coverage_{\alpha, \kappa}=\alpha - C_{\alpha, \kappa}$, where $C_{\alpha, \kappa}$ admits a local linear expansion with a closed-form coefficient.
\begin{lemma}
  \label{lem:coverage_asymptotic}
  Under the settings of Theorem~\ref{theorem:main}, we have as $n,d\to\infty$, $d/n\to\kappa\in(0,\kappa_0]$,
  \begin{align}\label{eqn:coverage_f_convergence}
    \coverage(\hat{f}) \gotop \coverage_{\alpha, \kappa}. 
  \end{align}
\end{lemma}

\begin{lemma}
  \label{lem:coverage_expansion}
  Under the same setting as Lemma~\ref{lem:coverage_asymptotic}, we further have
  \begin{equation}\label{eqn:coverage_expansion}
    \begin{aligned}
      & \quad \coverage_{\alpha, \kappa} = \alpha - C_{\alpha, \kappa} \\
      & = \alpha + (\phi_z(z_\alpha) \barb_0 + (1/2) \phi_z'(z_\alpha) \bartau_0^2) \kappa + o(\kappa).
    \end{aligned}
  \end{equation}
\end{lemma}

By Lemma~\ref{lem:coverage_expansion} and the definition of $\barb_0$ and $\bartau_0^2$ in~\eqref{eqn:def_p0}, the above coefficient in front of $\kappa$ can be simplified as
\begin{align*}
  & \quad \phi_z(z_\alpha) \barb_0 + (1/2) \phi_z'(z_\alpha) \bartau_0^2 \\
  & = \phi_z(z_\alpha)\cdot \frac{-\alpha(1-\alpha)\phi_z'(z_\alpha)-(2\alpha-1)\phi_z^2(z_\alpha)}{2\phi_z^3(z_\alpha)} + \frac{1}{2}\phi_z'(z_\alpha)\cdot \frac{\alpha(1-\alpha)}{\phi_z^2(z_\alpha)} \\
  & = -(\alpha - 1/2).
\end{align*}
This shows that $C_{\alpha, \kappa}=(\alpha-1/2)\kappa+o(\kappa)$, and in particular $C_{\alpha, \kappa}>0$ for all small $\kappa$ as $\alpha-1/2>0$. This proves Theorem \ref{theorem:main}.
\qed

The rest of this section is organized as follows. We prove Lemma~\ref{lem:scf_exist_unique} in Section~\ref{appendix:proof-c1-c2} (which requires analyzing a transformed system of equations and applying the implicit function theorem). In Section~\ref{appendix:connection-variational}, we connect the system of equations to a variational problem over four real variables. We then use this connection to prove Theorem~\ref{thm:ERM_limit} in Section~\ref{appendix:proof-ERM-limit}. Finally, we prove Lemma~\ref{lem:coverage_asymptotic} and Lemma~\ref{lem:coverage_expansion} in Section~\ref{appendix:proof-coverage-lemmas}.



\subsection{Proof of Lemma \ref{lem:scf_exist_unique} and Lemma~\ref{lem:local-linear-expansion}}
\label{appendix:proof-c1-c2}

\subsubsection{Analysis of system of equations~\eqref{equation:system}}
We first perform a change of variables. For any $(\bartau, \barlambda, \barb, \kappa) \in \overline \Omega \times (0, 1)$ where $\overline \Omega = \R_{\ge 0} \times \R_{\ge 0} \times \R$, we rewrite the system of equations~\eqref{equation:system} as 
\begin{equation}\label{eqn:def_bbF}
\bF(\bp; \kappa) = \bzero,
\end{equation}
where $\bp = (\bartau, \barlambda, \barb)$, $\bF(\bp; \kappa) \defeq (F_1(\bp; \kappa), F_2(\bp; \kappa), F_3(\bp; \kappa))$ in which
\begin{equation}
  \label{equation:transformed-system}
\begin{aligned}
      & F_1(\bartau, \barlambda, \barb; \kappa) \defeq \bartau^2 - \barlambda^2 \cdot \E\brac{ e_{\ell_{\barb\kappa+z_\alpha}}'(\bartau \sqrt{\kappa}G + Z; \barlambda\kappa)^2 }, \\
      & F_2(\bartau, \barlambda, \barb; \kappa) \defeq \bartau - \kappa^{-1/2}\barlambda \cdot \E\brac{e_{\ell_{\barb\kappa+z_\alpha}}'(\bartau \sqrt{\kappa}G + Z; \barlambda\kappa) G}, \\
      & F_3(\bartau, \barlambda, \barb; \kappa) \defeq \kappa^{-1}\E\brac{ e_{\ell_{\barb\kappa+z_\alpha}}'(\bartau \sqrt{\kappa}G + Z; \barlambda\kappa) }.
\end{aligned}
\end{equation}

Equation (\ref{eqn:def_bbF}) and the system~\eqref{equation:system} are equivalent up to a change of variables: For any fixed $\kappa$, any solution $(\tau_\star, \lambda_\star, b_\star)$ of Eq. (\ref{equation:system}) yields a solution $( \tau_\star / \kappa, \lambda_\star / \kappa, (b_\star - z_\alpha ) / \kappa, \kappa)$ of $\bF(\bp; \kappa) = \bzero$, and vice versa. Notice that this equivalence allows us to establish Lemma~\ref{lem:scf_exist_unique} and Lemma~\ref{lem:local-linear-expansion} by considering the transformed equation~\eqref{eqn:def_bbF}.



The following two auxiliary lemmas, which give a continuity analysis of the function $\bF$, are key to establishing Lemma~\ref{lem:scf_exist_unique} and Lemma~\ref{lem:local-linear-expansion}. These auxiliary lemmas are required for checking the conditions of the implicit function theorem. The proofs of these two lemmas are deferred to Section~\ref{appendix:proof-jacobian} and \ref{appendix:proof_lem_F_kappa_derivative} respectively. As a shorthand, we take
\begin{align*}
  \bp_0 = (\bartau_0, \barlambda_0, \barb_0),
\end{align*}
where $\bartau_0, \barlambda_0, \barb_0$ are defined in~\eqref{eqn:def_p0}.

\begin{lemma}\label{lem:Jacobian}
Let Assumption \ref{ass:noise_density} hold. Let $\bF$ be as defined in Eq. (\ref{eqn:def_bbF}). Then for any $\eps$ such that $\ball(\bp_0, 2 \eps) \subseteq \overline \Omega = \R_{\ge 0} \times \R_{\ge 0} \times \R$, there exists a continuous matrix function $\bJ: \ball(\bp_0, \eps) \to \R^{3 \times 3}$ with 
\begin{equation}\label{eqn:Jacobian_lower_bound}
\sigma_{\min}(\bJ(\bp_0)) > 0, 
\end{equation}
and
\begin{equation}\label{eqn:Jacobian_convergence}
\lim_{\kappa \to 0} \sup_{\bp \in \ball(\bp_0, \eps)}  \Big\| \nabla_{\bp} \bF(\bp, \kappa) - \bJ(\bp) \Big\|_{\op} = 0. 
\end{equation}
\end{lemma}

\begin{lemma}\label{lem:F_kappa_derivative_solution_kappa_0}
Let Assumption \ref{ass:noise_density} hold. Let $\bF$ be as defined in Eq. (\ref{eqn:def_bbF}). Then for any $\eps$ such that $\ball(\bp_0, 2 \eps) \subseteq \overline \Omega = \R_{\ge 0} \times \R_{\ge 0} \times \R$, there exists two continuous vector functions $\bF_0,\bg: \ball(\bp_0, \eps) \to \R^3$ such that
\[
\begin{aligned}
\lim_{\kappa \to 0} \sup_{\bp \in \ball(\bp_0, \eps)} \Big\| \bF(\bp, \kappa) - \bF_0(\bp) \Big\|_2 =&~ 0, \\
\lim_{\kappa \to 0} \sup_{\bp \in \ball(\bp_0, \eps)} \Big\| \partial_\kappa \bF(\bp, \kappa) - \bg(\bp) \Big\|_2 =&~ 0. 
\end{aligned}
\]
Moreover, we have 
\[
\lim_{\kappa \to 0+} \bF(\bp_0, \kappa) =  \bF_0(\bp_0) = \bzero.
\] 
\end{lemma}
By Lemma \ref{lem:Jacobian} and \ref{lem:F_kappa_derivative_solution_kappa_0}, we can continuously extend the function $\bF$ to the region $\ball(\bp_0, \eps) \times [0, \kappa_0)$ for some small $\kappa_0$, such that $\bF(\bp, \kappa)$ is continuously differentiable in the same region. Moreover, by Lemma \ref{lem:F_kappa_derivative_solution_kappa_0}, we have $\bF(\bp_0, 0) = \lim_{\kappa \to 0} \bF(\bp_0, \kappa) = 0$. Finally, by Lemma \ref{lem:Jacobian}, we have $\sigma_{\min}(\nabla_{\bp} \bF(\bp_0, 0)) > 0$.

\subsubsection{Proof of Lemma \ref{lem:Jacobian}}
\label{appendix:proof-jacobian}

For any $\bp = (\bartau, \barlambda, \barb) \in \overline \Omega = \R_{\ge 0} \times \R_{\ge 0} \times \R$, we define a continuous matrix function $\bJ:\overline{\Omega} \to \R^{3\times 3}$ by
\begin{align*}
 \bJ(\bp) =
  \begin{pmatrix}
    2\bartau & -2 \barlambda \alpha (1 - \alpha) & 0 \\
    1 - \barlambda \phi_z(z_\alpha) & - \bartau \phi_z(z_\alpha) & 0 \\
    - \bartau \phi_z'(z_\alpha) & (1 - 2 \alpha) \phi_z(z_\alpha) & - \phi_z(z_\alpha)
  \end{pmatrix}.
\end{align*}
Evaluating $\bJ(\bp_0)$ (recall $\bp_0$ is defined in Eq. (\ref{eqn:def_p0})), we have 
\[
 \bJ(\bp_0) =
  \begin{pmatrix}
    \frac{2 \sqrt{\alpha (1 - \alpha)}}{\phi_z(z_\alpha)} & -\frac{2\alpha (1 - \alpha)}{\phi_z(z_
    \alpha)} & 0 \\
   0 & - \sqrt{\alpha(1 - 
   \alpha)} & 0 \\
    - \frac{\sqrt{\alpha(1 - \alpha)}}{\phi_z(z_\alpha)} \phi_z'(z_\alpha) & (1 - 2 \alpha) \phi_z(z_\alpha) & - \phi_z(z_\alpha)
  \end{pmatrix}.
\]
Since we have assumed that $\phi_z(z_\alpha) \neq 0$, it is easy to see that $\det(\bJ(\bp_0)) = - 2 \alpha (1 - \alpha ) \neq 0$. This proves Eq. (\ref{eqn:Jacobian_lower_bound}). 

We next prove Eq. (\ref{eqn:Jacobian_convergence}). Recall that the definition of $\bF = (F_1, F_2, F_3)$ as given in Eq. (\ref{equation:transformed-system}),  
by the calculus of $e_{\ell_b}$ as in Section \ref{sec:Moreau_calculus}, we have 
\[
\begin{aligned}
F_1(\bp; \kappa) =&~ \bartau^2 - \E_G\Big\{ \frac{1}{\kappa^2} \int_{[\oG_-, \oG_+]} (z - \oG)^2 \phi_z(z) \de z +  \barlambda^2 \alpha^2 [1 - \Phi_z(\oG_+)] +  \barlambda^2 (1 - \alpha)^2  \Phi_z(\oG_-) \Big\}, \\
F_2(\bp; \kappa) =&~ \bartau -   \kappa^{-1/2} \E_G \Big\{ \frac{1}{\kappa} \int_{[\oG_-, \oG_+]} (z - \oG) G \phi_z(z) \de z + \barlambda \alpha [(1 - \Phi_z(\oG_+))G] -   \barlambda (1 - \alpha)  \Phi_z(\oG_-) G \Big\}, \\
F_3(\bp; \kappa) =&~ \kappa^{-1} \E_G\Big\{ \frac{1}{\barlambda \kappa} \int_{[\oG_-, \oG_+]} (z - \oG) \phi_z(z) \de z + \alpha [1 - \Phi_z(\oG_+)] -   (1 - \alpha) \Phi_z(\oG_-) \Big\}, \\
\end{aligned}
\]
where 
\begin{equation}\label{eqn:def_bar_G}
\begin{aligned}
\oG \equiv&~ z_\alpha + \kappa \barb - G \bartau \sqrt{\kappa} , \\
\oG_+ \equiv&~ z_\alpha + \kappa \barb + \alpha \kappa \barlambda - G \bartau \sqrt \kappa, \\
\oG_- \equiv&~ z_\alpha + \kappa \barb - (1 - \alpha) \kappa \barlambda - G \bartau \sqrt \kappa.
\end{aligned}
\end{equation}
Using the smoothness property of $\phi_z$, with some calculus, we have
\[
\begin{aligned}
\lim_{\kappa \to 0} \partial_{\bartau} F_1(\bp; \kappa ) =&~ 2 \bartau, \\
\lim_{\kappa \to 0} \partial_{\bartau} F_2(\bp; \kappa ) =&~ 1 - \barlambda \phi_z(z_\alpha), \\
\lim_{\kappa \to 0} \partial_{\bartau} F_3(\bp; \kappa ) =&~ - \bartau \phi_z'(z_\alpha),\\
\lim_{\kappa \to 0} \partial_{\barlambda} F_1(\bp; \kappa ) =&~ -2 \barlambda \alpha (1 - \alpha), \\
\lim_{\kappa \to 0} \partial_{\barlambda} F_2(\bp; \kappa ) =&~ - \bartau \phi_z(z_\alpha), \\
\lim_{\kappa \to 0} \partial_{\barlambda} F_3(\bp; \kappa ) =&~ (1 - 2 \alpha) \phi_z(z_\alpha), \\
\lim_{\kappa \to 0} \partial_{\barb} F_1(\bp; \kappa ) =&~ 0, \\
\lim_{\kappa \to 0} \partial_{\barb} F_2(\bp; \kappa ) =&~ 0, \\
\lim_{\kappa \to 0} \partial_{\barb} F_3(\bp; \kappa ) =&~ - \phi_z(z_\alpha). \\
\end{aligned}
\]
This proves that $\lim_{\kappa \to 0} \nabla_{\bp} \bF(\bp; \kappa) = \bJ(\bp)$. With some more refined analysis, it is easy to see that the convergence above is uniform over $\bp \in \ball(\bp_0, \eps)$ for small $\eps$. This proves the lemma.
\qed

\subsubsection{Proof of Lemma \ref{lem:F_kappa_derivative_solution_kappa_0}}\label{appendix:proof_lem_F_kappa_derivative}

In this proof, we follow the same notations with the proof of Lemma \ref{lem:Jacobian} as in Section \ref{appendix:proof_lem_F_kappa_derivative}. 

For any $(\bartau, \barlambda, \barb, \kappa) \in \overline \Omega \times (0, \kappa_0)$ where $\overline \Omega = \R_{\ge 0} \times \R_{\ge 0} 
\times \R$, we define
\begin{equation}\label{eqn:def_f_123}
\begin{aligned}
f_1(\bp, \kappa) =&~ \E\brac{ e_{\ell_{\barb\kappa+z_\alpha}}'(\bartau \sqrt{\kappa}G + Z; \barlambda\kappa)^2 }, \\
f_2(\bp, \kappa) =&~ \E\brac{e_{\ell_{\barb\kappa+z_\alpha}}''(\bartau \sqrt{\kappa}G + Z; \barlambda\kappa)}, \\
f_3(\bp, \kappa) =&~ \E\brac{ e_{\ell_{\barb\kappa+z_\alpha}}'(\bartau \sqrt{\kappa}G + Z; \barlambda\kappa) }.  
\end{aligned}
\end{equation}
By the definition of $F_1, F_2, F_3$ as in Eq. (\ref{equation:transformed-system}), we have
\begin{equation}
\begin{aligned}
F_1(\bp, \kappa) =&~ \bartau^2 - \barlambda^2 f_1(\bp, \kappa), \\
F_2(\bp, \kappa) =&~ \bartau - \bartau \barlambda f_2(\bp, \kappa), \\
F_3(\bp, \kappa) =&~ \kappa^{-1} f_3(\bp, \kappa).  
\end{aligned}
\end{equation}
Then, Lemma \ref{lem:F_kappa_derivative_solution_kappa_0} holds as long as we show that there exists continuous functions $\bT(\bp) = (T_{1}(\bp), T_{2}(\bp), T_{3}(\bp))$ and $\bg(\bp) = (g_1(\bp), g_2(\bp), g_3(\bp))$ such that
\begin{align}
f_1(\bp, \kappa) =&~  T_{1}(\bp) + o(1), \label{eqn:f1_expansion_in_lemma} \\
\partial_\kappa f_1(\bp, \kappa) =&~  - \barlambda^{-2} g_1(\bp) + o(1), \label{eqn:f1_prime_expansion_in_lemma} \\
f_2(\bp, \kappa) =&~ T_{2}(\bp) + o(1), \\
\partial_\kappa f_2(\bp, \kappa) =&~ - (\bartau \barlambda)^{-1} g_2(\bp) + o(1), \\
f_3(\bp, \kappa) =&~ o(1), \\
\partial_\kappa f_3(\bp, \kappa) =&~ T_{3}(\bp) + o(1), \\
\partial_\kappa^2 f_3(\bp, \kappa) =&~ g_3(\bp) + o(1), 
\end{align}
where the $o(1)$ terms convergence to $0$ uniformly over $\bp \in \ball(\bp_0, \eps)$ as $\kappa \to 0+$. Moreover, we need 
\begin{align}
T_{1}(\bp_0) =&~  \bartau_0^2 / \barlambda_0^2, \label{eqn:F01_in_lemma} \\
T_{2}(\bp_0) =&~ 1/\barlambda_0, \\
T_{3}(\bp_0) =&~ 0.\label{eqn:F03_in_lemma}
\end{align}

We first prove Eq. (\ref{eqn:f1_expansion_in_lemma}), (\ref{eqn:f1_prime_expansion_in_lemma}) and (\ref{eqn:F01_in_lemma}). First, we have (c.f. Eq. (\ref{eqn:def_bar_G})) 
\[
\begin{aligned}
\lim_{\kappa \to 0+} f_1(\bp, \kappa) =&~ \lim_{\kappa \to 0+} \E\Big[ \frac{1}{\barlambda^2 \kappa^2} \int_{[\oG_-, \oG_+]} (z - \oG)^2 \phi_z(z) \de z + \alpha^2 [1 - \Phi_z(\oG_+)] +  (1 - \alpha)^2  \Phi_z(\oG_-)\Big]  \\
=&~ \alpha^2 (1 - \Phi_z(z_\alpha)) + (1 - \alpha)^2 \Phi_z(z_\alpha)= \alpha (1 - \alpha) = \bartau_0^2 / \barlambda_0^2. 
\end{aligned}
\]
where the last equality is by the definition in Eq. (\ref{eqn:def_p0}). Further, by smoothness of the density $\phi_z$, and the fact that the neighborhood $\ball(\bp_0, \eps)$ is bounded, this convergence is uniform over $\bp = (\bartau, \barlambda, \barb) \in \ball(\bp_0, \eps)$. This proves Eq. (\ref{eqn:f1_expansion_in_lemma}) and (\ref{eqn:F01_in_lemma}). 

Moreover, we have
\[
\begin{aligned}
&~\partial_\kappa f_1(\bp, \kappa) \\
=&~ \E\Big[ - \frac{2}{\barlambda^2 \kappa^3} \int_{[\oG_-, \oG_+]} (z - \oG)^2 \phi_z(z) \de z \\
&~ +  \frac{1}{\barlambda^2 \kappa^2} (\oG_+ - \oG)^2 \phi_z(\oG_+)  (\bar b + \alpha \barlambda - G \bartau / (2 \sqrt{\kappa})) \\
&~ -  \frac{1}{\barlambda^2 \kappa^2} (\oG_- - \oG)^2 \phi_z(\oG_-)  (\bar b - (1 - \alpha) \barlambda - G \bartau / (2 \sqrt{\kappa})) \\
&~ - \alpha^2 \phi_z(\oG_+) (\bar b + \alpha \barlambda - G \bartau / (2 \sqrt{\kappa})) +  (1 - \alpha)^2  \phi_z(\oG_-) (\bar b - (1- \alpha) \barlambda - G \bartau / (2 \sqrt{\kappa}))  \Big]\\
=&~ \E\Big[ - \frac{2}{\barlambda^2 \kappa^3} \int_{[\oG_-, \oG_+]} (z - \oG)^2 \phi_z(z) \de z \Big], \\
\end{aligned}
\]
where the last inequality is by Stein's identity for $Z \sim \cN(0, 1)$ and a consequence of many cancellation happening. So this gives 
\[
\begin{aligned}
\lim_{\kappa \to 0+} \partial_\kappa f_1(\bp, \kappa) =&~ 
- \frac{2}{3 \barlambda^2 } \Big[ \alpha^2  - (1 - \alpha)^3 \Big]  \phi_z(z_\alpha). 
\end{aligned}
\]
Again, by the smoothness of $\phi_z$, and the fact that the neighborhood $\ball(\bp_0, \eps)$ is bounded, this convergence is uniform over $\bp = (\bartau, \barlambda, \barb) \in \ball(\bp_0, \eps)$. This proves Eq. (\ref{eqn:f1_prime_expansion_in_lemma}). The proof of other equations within~\eqref{eqn:f1_expansion_in_lemma} to~\eqref{eqn:F03_in_lemma} follow from similar continuity arguments. This proves Lemma~\ref{lem:F_kappa_derivative_solution_kappa_0}.
\qed

\subsubsection{Proof of Lemma \ref{lem:scf_exist_unique} and Lemma~\ref{lem:local-linear-expansion}}

We consider the function $\bF$ defined in~\eqref{equation:transformed-system}. First, by Lemma~\ref{lem:F_kappa_derivative_solution_kappa_0}, we have $\bF(\bp_0, 0_+)=\bzero$. Further, by Lemma~\ref{lem:Jacobian} and~\ref{lem:F_kappa_derivative_solution_kappa_0}, the conditions in the Implicit Function Theorem (Lemma~\ref{lem:implicit_function}) are satisfied, from which we can conclude that there exists $\kappa_0>0$ and a continuously differentiable path $\{ \bp(\kappa)=(\bartau(\kappa), \barlambda(\kappa), \barb(\kappa)) : \kappa \in [0, \kappa_0) \} \subset \ball(\bp_0, \eps)$, such that $\bF(\bp(\kappa), \kappa) = 0$ for any $\kappa \in [0, \kappa_0)$. Therefore, the set of variables
\begin{align*}
  \paren{\tau_\star(\kappa), \lambda_\star(\kappa), b_\star(\kappa)} = \paren{\bartau(\kappa)\cdot \kappa, ~\barlambda(\kappa)\cdot \kappa,~ z_\alpha+\barb\cdot \kappa},
\end{align*}
is a unique solution to the original system of equations~\eqref{equation:system} by the equivalence between system~\eqref{equation:system} and system~\eqref{eqn:def_bbF} under this change of variables. This proves Lemma~\ref{lem:scf_exist_unique}.

In order to prove Lemma~\ref{lem:local-linear-expansion} (the local linear expansion),
it suffices to prove that $\bp(\kappa) \to \bp_0=(\bartau_0,\barlambda_0,\barb_0)$. This was already implied by the continuity of $\bp(\kappa)$ w.r.t. $\kappa$ as stated above.
\qed




\subsection{Connection between system of equations~\eqref{equation:system} and a variational problem}
\label{appendix:connection-variational}

Define
\begin{align}
  \label{equation:d4}
  D(\tau, b, \tau_g, \beta) \equiv \Big[ \frac{\beta \tau_g}{2} + \frac{1}{\kappa} \E_{(G, Z) \sim \cN(0, 1) \times P_z}\brac{ e_{\ell_b}(\tau G + Z; \tau_g/\beta)} -  \tau\beta \Big]. 
\end{align}
The $D$ defined above is strictly convex-concave as stated in the following lemma.
\begin{lemma}[Strict convexity-concavity]
  \label{lem:analysis_D_new}
  Suppose $\kappa\in(0,1)$. Then for any $(\tau, b, \tau_g, \beta) \in \R_{> 0} \times \R \times \R_{> 0} \times \R_{> 0}$, the function $D$ defined in~\eqref{equation:d4} is strictly convex in $(\tau, b, \tau_g)$ ($\grad^2_{\tau, b, \tau_g} D\succ \bzero$), and strictly concave in $\beta$.
\end{lemma}
\begin{proof}[Proof of Lemma \ref{lem:analysis_D_new}]
Define 
\[
E(\tau, b, \tau_g, \beta) \equiv \E_{(G, Z) \sim \cN(0, 1) \times P_z}\brac{ e_{\ell_b}(\tau G + Z; \tau_g/\beta)}. 
\]
We write in short $\partial_x e = \partial_x e_{\ell_b}(\tau G + Z; \tau_g/\beta)$ and $\partial_x^2 e = \partial_x \partial_x e_{\ell_b}(\tau G + Z; \tau_g/\beta)$. Then by Eq. (\ref{eqn:derivative_e_function}), we have 
\[
\begin{aligned}
\partial_\tau E(\tau, b, \tau_g, \beta) \equiv&~ \E\brac{ \partial_x e  \cdot G}, \\
\partial_b E(\tau, b, \tau_g, \beta) \equiv&~ - \E\brac{ \partial_x e}, \\
\partial_{\tau_g} E(\tau, b, \tau_g, \beta) \equiv&~ - \frac{1}{2 \beta} \E\brac{ (\partial_x e)^2}, \\
\partial_\beta E(\tau, b, \tau_g, \beta) \equiv&~ \frac{\tau_g}{2 \beta^2} \E\brac{ (\partial_x e)^2}. \\
\end{aligned}
\]
By Eq. (\ref{eqn:second_derivative_e_function}), for any $(\tau, b, \tau_g, \beta) \in \R_{> 0} \times \R \times \R_{> 0} \times \R_{> 0}$, we have 
\[
\begin{aligned}
\partial_\beta^2 E = - \frac{\tau_g}{\beta^3} \E[(\partial_x e)^2] + \frac{\tau_g^2}{\beta^4} \E[(\partial_x e)^2 \partial_x^2 e] = - \frac{\tau_g}{\beta^3} \E[(\partial_x e)^2 \ones\{ \prox_{\ell_b}(\tau G + Z) \neq b \}] < 0. 
\end{aligned}
\]
This gives $\partial_\beta^2 D = \kappa^{-1} \partial_\beta^2 E < 0$, so that $D$ is strictly concave in $\beta$ (for any fixed $(\tau, b, \tau_g)$).

By Eq. (\ref{eqn:second_derivative_e_function}) again, we have 
\[
\begin{aligned}
\nabla_{(\tau, b, \tau_g)}^2 E =&~ \E \begin{bmatrix}
G^2 \cdot  \partial_x^2 e  & -  G \cdot  \partial_x^2 e & - \beta^{-1} \partial_x e \cdot  G \cdot  \partial_x^2 e \\
 - G \cdot  \partial_x^2 e  & \partial_x^2 e & \beta^{-1} \partial_x e \cdot   \partial_x^2 e \\
  - \beta^{-1} \partial_x  e \cdot  G \cdot  \partial_x^2 e & \beta^{-1} \partial_x e \cdot  \partial_x^2 e  & \beta^{-2} (\partial_x e)^2 \cdot  \partial_x^2 e \\
\end{bmatrix}\\
=&~   \frac{\beta}{\tau_g} \E [ \ones\{ \prox_{\ell_b}(\tau G + Z; \tau_g / \beta) \neq b \} \cdot \bu \bu^\top]. 
\end{aligned}
\]
where $\bu = (G, -1, \beta^{-1} \partial_x e)$. Note that there exists $(G_1, Z_1)$, $(G_2, Z_2)$ and $(G_3, Z_3)$ such that $\prox_{\ell_b}(\tau G_1 + Z_1; \tau_g / \beta), \prox_{\ell_b}(\tau G_2 + Z_2; \tau_g / \beta), \prox_{\ell_b}(\tau G_3 + Z_3; \tau_g / \beta) \neq b$, and
\[
\begin{aligned}
\begin{bmatrix}
G_1 & -1 & \beta^{-1} \partial_x e_{\ell_b}(\tau G_1 + Z_1; \tau_g / \beta)\\
G_2 & -1 & \beta^{-1} \partial_x e_{\ell_b}(\tau G_2 + Z_2; \tau_g / \beta)\\
G_3 & -1 & \beta^{-1} \partial_x e_{\ell_b}(\tau G_3 + Z_3; \tau_g / \beta)\\
\end{bmatrix}
\end{aligned}
\]
is full rank. By Lemma \ref{lem:PD_condition}, we have $\nabla_{(\tau, b, \tau_g)}^2 E \succ 0$. Note that $\nabla_{(\tau, b, \tau_g)}^2 D = \kappa^{-1} \nabla_{(\tau, b, \tau_g)}^2 E \succ 0$, so that $D$ is strictly convex in $(\tau, b, \tau_g)$ (for any fixed $\beta$). This proves the lemma. 
\end{proof}

We now characterize a min-max variational problem associated with the function $D$, and show that it has a unique solution for small $\kappa$, and the solution is related to the solution of the system of equations~\eqref{equation:system}.
\begin{lemma}[Characterization of variational problem]
  \label{lemma:variational}
  Consider the following variational problem in four variables over the function $D$ defined in~\eqref{equation:d4}:
  \begin{equation}
    \label{equation:v4}
    \begin{aligned}
      & \quad \inf_{\tau>0,b\in\R,\tau_g>0} \sup_{\beta>0} ~D(\tau, b, \tau_g, \beta) \\
      & = \inf_{\tau>0,b\in\R,\tau_g>0} \sup_{\beta>0} ~\Big[ \frac{\beta \tau_g}{2} + \frac{1}{\kappa} \E_{(G, Z) \sim \cN(0, 1) \times P_z}\brac{ e_{\ell_b}(\tau G + Z; \tau_g/\beta)} -  \tau\beta \Big].
    \end{aligned}
  \end{equation}
  For all sufficiently small $\kappa\in(0,\kappa_0]$, there exists a unique solution $(\wt{\tau}_\star, \wt{b}_\star, \wt{\tau}_{g,\star}, \wt{\beta}_\star)$ (which depends on $\kappa$) to problem~\eqref{equation:v4}. This solution is related to the solution $(\tau_\star(\kappa), \lambda_\star(\kappa), b_\star(\kappa))$ of~\eqref{equation:system} as
  \begin{align}
    \label{equation:solution-transition}
    \wt{\tau}_\star=\wt{\tau}_{g,\star}=\tau_\star(\kappa),~~\wt{\beta}_\star = \tau_\star(\kappa)/\lambda_\star(\kappa),~~\wt{b}_\star = b_\star(\kappa).
  \end{align}
  Further, for some positive $\eps>0$, for any $b'\in[b_\star-\eps, b_\star+\eps]$, the following variational problem in three variables
  \begin{equation}
    \label{equation:v3}
    \begin{aligned}
      & \quad \inf_{\tau>0,\tau_g>0} \sup_{\beta>0} ~D(\tau, b', \tau_g, \beta) \\
      & = \inf_{\tau>0,b\in\R,\tau_g>0} \sup_{\beta>0} ~\Big[ \frac{\beta \tau_g}{2} + \frac{1}{\kappa} \E_{(G, Z) \sim \cN(0, 1) \times P_z}\brac{ e_{\ell_{b'}}(\tau G + Z; \tau_g/\beta)} -  \tau\beta \Big]
    \end{aligned}
  \end{equation}
  has a unique solution within $\R_{>0}^3$.
\end{lemma}
\begin{proof}[Proof of Lemma~\ref{lemma:variational}]

Calculating the derivatives of $D(\tau, b, \tau_g, \beta)$, we get 
\[
\begin{aligned}
\partial_\tau D(\tau, b, \tau_g, \beta) =&~ \kappa^{-1} \E[G e_{\ell_b}'(\tau G + Z; \tau_g / \beta)] - \beta, \\
\partial_b D(\tau, b, \tau_g, \beta) =&~ -\kappa^{-1} \E[e_{\ell_b}'(\tau G + Z; \tau_g / \beta)], \\
\partial_{\tau_g} D(\tau, b, \tau_g, \beta) =&~ \beta / 2 - \frac{1}{2 \kappa \beta} \E[e_{\ell_b}'(\tau G + Z; \tau_g / \beta)^2], \\
\partial_{\beta} D(\tau, b, \tau_g, \beta) =&~ \tau_g / 2 - \tau + \frac{\tau_g}{2 \kappa \beta^2} \E[e_{\ell_b}'(\tau G + Z; \tau_g / \beta)^2]. \\
\end{aligned}
\]
By Lemma \ref{lem:scf_exist_unique}, there exists $\kappa_0 > 0$ such that for any $\kappa \in (0, \kappa_0]$, there exists a unique solution $(\tau_\star(\kappa), \lambda_\star(\kappa), b_\star(\kappa))$ of Eq. (\ref{equation:system}). Plugging in $(\tau, b, \tau_g, \beta) = (\tau_\star(\kappa), b_\star(\kappa), \tau_\star(\kappa), \tau_\star(\kappa) / \lambda_\star(\kappa))$ into the derivatives above and using Eq. (\ref{equation:system}), we get $\nabla_{(\tau, b, \tau_g, \beta)} D(\tau_\star(\kappa), b_\star(\kappa), \tau_\star(\kappa), \tau_\star(\kappa) / \lambda_\star(\kappa)) = 0$. This proves that $(\wt{\tau}_\star, \wt{b}_\star, \wt{\tau}_{g,\star}, \wt{\beta}_\star) = (\tau_\star(\kappa), b_\star(\kappa), \tau_\star(\kappa), \tau_\star(\kappa) / \lambda_\star(\kappa))$ is a stationary point of $D$.

Since $D$ is jointly strictly convex in $(\tau, b, \tau_g)$ and strictly concave in $\beta$ as stated in Lemma \ref{lem:analysis_D_new}, we get 
\[
\begin{aligned}
&~ \inf_{\tau>0,b\in\R,\tau_g>0} \sup_{\beta>0} ~D(\tau, b, \tau_g, \beta) \le \sup_{\beta > 0} D(\wt{\tau}_\star, \wt{b}_\star, \wt{\tau}_{g,\star}, \beta) = D(\wt{\tau}_\star, \wt{b}_\star, \wt{\tau}_{g,\star}, \wt{\beta}_\star), \\
&~ \inf_{\tau>0,b\in\R,\tau_g>0} \sup_{\beta>0} ~D(\tau, b, \tau_g, \beta) \ge  \inf_{\tau > 0, b \in \R, \tau_g > 0} D(\tau, b, \tau_g, \wt{\beta}_\star) = D(\wt{\tau}_\star, \wt{b}_\star, \wt{\tau}_{g,\star}, \wt{\beta}_\star). \\
\end{aligned}
\]
This proves that $(\wt{\tau}_\star, \wt{b}_\star, \wt{\tau}_{g,\star}, \wt{\beta}_\star)$ is a solution of the variational problem (\ref{equation:v4}). By the strict convexity-concavity property of $D$ again, the solution of the variational problem (\ref{equation:v4}) is unique. Finally, the existence and uniqueness of the solution of $\inf_{\tau>0,\tau_g>0} \sup_{\beta>0} ~D(\tau, b', \tau_g, \beta)$ for $b' \in [b_\star - \eps, b_\star + \eps]$ follows from similar arguments. 
\end{proof}

\subsection{Proof of Theorem~\ref{thm:ERM_limit}}
\label{appendix:proof-ERM-limit}

\paragraph{Preliminary: the asymptotic limit fixed $b$ via CGMT}
For any convex function $\ell: \R \to \R$, we define notation
\[
\ell_+'(v) \equiv \sup_{s \in \partial \ell(v)} \vert s \vert. 
\]
For $\tau > 0$, we define (with some abuse of notation)
\begin{align}
  \label{equation:d1}
D(\tau) \equiv \inf_{\tau_g > 0} \sup_{ \beta > 0} \Big[ \frac{\beta \tau_g}{2} + \frac{1}{\kappa} \E_{(G, Z) \sim \cN(0, 1) \times P_z}\brac{ e_{\ell}(\tau G + Z; \tau_g/\beta)} -  \tau\beta \Big]. 
\end{align}
The following proposition is by \cite[Theorem 4.1]{thrampoulidis2018precise}, which uses the Convex Gaussian Comparison Theorem (CGMT). 
\begin{proposition}[A simplification of Theorem 4.1 in \cite{thrampoulidis2018precise} up to model rescaling]\label{prop:CGMT}
Let $\ell$ be a closed proper convex function and $P_z$ be a distribution on the real line satisfying 
\begin{itemize}
\item $\E_{(G, Z) \sim \cN(0, 1) \times P_z}[\vert \ell_+'(c G + Z) \vert^2] < \infty$, for all $c \in \R$; 
\item $\sup_{v \in \R} \vert \ell_+'(v) \vert < \infty$. 
\end{itemize}
Further assume that the set $\argmin_{\tau} D(\tau)$ is bounded for the function $D$ defined in~\eqref{equation:d1}. Then $D$ has a unique minimizer $\tau_\star > 0$. Moreover, in the limit $n, d \to \infty$ and $d / n \to \kappa$, we have 
\[
\min_{\wb} \frac{1}{n} \sum_{i = 1}^n \ell(y_i - \< \xb_i, \wb\> ) \stackrel{p}{\to} \min_{\tau} D(\tau).  
\]
Furthermore, for any $\eps > 0$, defining $S_\eps \equiv \{ \wb:  \vert \| \wb - \wb_\star \|_2^2 - \tau_\star^2 \vert \le \eps \}$, there exists $\delta > 0$ such that 
\[
\min_{\wb \in S_\eps^c} \frac{1}{n} \sum_{i = 1}^n \ell(y_i - \< \xb_i, \wb\> ) \stackrel{p}{\to} \min_{\tau} D(\tau) + \delta.  
\]
As a consequence, for any empirical risk minimizer $\hat \wb$ satisfying
\[
\hat \wb \in \argmin_{\wb} \frac{1}{n} \sum_{i = 1}^n \ell(y_i - \< \xb_i, \wb\> ),
\]
we have 
\[
 \| \hat \wb - \wb_\star \|_2^2 \stackrel{p}{\to} \tau_\star^2.
\]
\end{proposition}


We are now ready to prove Theorem~\ref{thm:ERM_limit}.
\begin{proof}[Proof of Theorem \ref{thm:ERM_limit}]
  We define (with some abuse of notation)
  \begin{align}\label{equation:d2}
    D(\tau, b) \equiv \inf_{\tau_g > 0} \sup_{ \beta > 0} \Big[ \frac{\beta \tau_g}{2} + \frac{1}{\kappa} \E_{(G, Z) \sim \cN(0, 1) \times P_z}\brac{ e_{\ell_b^\alpha}(\tau G + Z; \tau_g/\beta)} -  \tau\beta \Big]. 
  \end{align}

\noindent
{\bf Step 1. Show that $\hat b \gotop b_\star$. }
For any fixed $b\in\R$, define the associated minimum empirical risk (over $\wb\in\R^d$) as
\[
L_n(b) \equiv \min_\wb \what R_n(\wb, b). 
\]
Notice that $\what{b} = \argmin_{b\in\R} L_n(b)$.
Let $(\tau_\star, \kappa_\star, b_\star)$ be defined as in Lemma~\ref{lem:scf_exist_unique} (as well as Lemma~\ref{lemma:variational}). By Lemma~\ref{lemma:variational}, there exists some $\eps>0$ such that for any fixed $b\in[b_\star-\eps, b_\star+\eps]$, we have $\argmin_\tau D(\tau)$ is a singleton. Therefore the conditions of Proposition~\ref{prop:CGMT} is satisfied, from which we conclude that
\[
L_n(b) \stackrel{p}{\to} \min_{\tau} D(\tau, b). 
\]
Now, observe that $\min_\tau D(\tau, b)=\min_{\tau, \tau_g} \max_{\beta} D(\tau, b, \tau_g, \beta)$ is strictly convex in $b$ (this is because $D(\tau, b, \tau_g, \beta)$ has a positive definite Hessian w.r.t. $(\tau, b, \tau_g)$ at any $(\tau, b, \tau_g, \beta)$ by Lemma~\ref{lem:analysis_D_new}).
Then for any $\eps > 0$, there exists $\delta > 0$ such that 
\[
\min_{\tau} D(\tau, b_\star + \eps) \ge \min_{\tau} D(\tau, b_\star) + \delta,~~~~~\min_{\tau} D(\tau, b_\star - \eps) \ge \min_{\tau} D(\tau, b_\star) + \delta. 
\]
As a consequence, with probability going to $1$, we have the event 
\[
\{ L_n(b_\star + \eps) > L_n(b_\star) + \delta / 2, ~~~~ L_n(b_\star - \eps) > L_n(b_\star) + \delta / 2 \}. 
\]
Furthermore, since $L_n(b)$ is a convex function in $b$, this implies that, with probability going to $1$, we have $\vert \hat b - b_\star \vert \le \eps$. Note that this is for any $\eps > 0$. This proves that $\hat b \stackrel{p}{\to} b_\star$. 

\noindent
{\bf Step 2. Show that $\| \hat \wb - \wb_\star \|_2^2 \gotop \tau_\star^2$. } By Proposition \ref{prop:CGMT}, for any $\eps > 0$, there exists $\delta > 0$ such that 
\[
\min_{\wb \in S_\eps^c} \what R_n(\wb, b_\star) \stackrel{p}{\to} \min_{\tau} D(\tau, b_\star) + \delta.  
\]
where $S_\eps \equiv \{ \wb:  \vert \| \wb - \wb_\star \|_2^2 - \tau_\star^2 \vert \le \eps \}$. 

Furthermore, note that $\ell^\alpha(t) = -(1-\alpha)t\indic{t \le 0} + \alpha t\indic{t > 0}$ is a $1$-Lipschitz function in $t$, this gives
\[
\sup_{\wb} \Big\vert \what R_n(\wb, b_1) - \what R_n(\wb, b_2) \Big\vert \le \vert b_1 - b_2 \vert
\]
As a consequence, we have 
\[
\min_{\wb \in S_\eps^c} \what R_n(\wb, \hat b) \ge \min_{\wb \in S_\eps^c} \what R_n(\wb, b_\star) - \vert \hat b - b_\star \vert \stackrel{p}{\to} \min_{\tau} D(\tau, b_\star) + \delta. 
\]
In the mean time, by Proposition \ref{prop:CGMT}, we have 
\[
\min_{\wb} \what R_n(\wb, \hat b) \le \min_{\wb} \what R_n(\wb, b_\star) + \vert \what b - b_\star \vert \stackrel{p}{\to} \min_{\tau} D(\tau, b_\star). 
\]
This implies that, with probability approaching $1$, we have
\begin{align*}
  \min_{\wb \in S_\eps^c} \what R_n(\wb, b_\star) \ge \min_{\tau} D(\tau, b_\star) + 2\delta/3 ~~~{\rm and}~~~\min_{\wb} \what R_n(\wb, \hat b) \le \min_{\tau} D(\tau, b_\star) + \delta/3.
\end{align*}
On this event we have $\hat \wb \in S_\eps$. Note that this is for any $\eps > 0$. This proves that $\| \hat \wb - \wb_\star \|_2^2 \gotop \tau_\star^2$.
\end{proof}

%
%
%

\subsection{Proof of Lemma~\ref{lem:coverage_asymptotic} and Lemma~\ref{lem:coverage_expansion}}
\label{appendix:proof-coverage-lemmas}

Recall that 
\begin{align*}
\coverage(\hat{f}) = \P_{(\xb, y)}\paren{y \le \hat{\wb}^\top\xb + \hat{b}} = \E_{G\sim\normal(0, 1)}\brac{ \Phi_z\paren{\ltwo{\hat{\wb} - \wb_\star}G + \hat{b}}}.
\end{align*}
Eq. (\ref{eqn:coverage_f_convergence}) is simply by the fact that $T(\tau, b; G) \equiv \Phi_z(\tau G + b)$ is a continuous function in $(\tau, b)$, by Theorem \ref{thm:ERM_limit}, and by the dominant convergence theorem. This proves Lemma~\ref{lem:coverage_asymptotic}.

Furthermore, by Taylor expansion, we have 
\[
\begin{aligned}
&~\coverage_{\alpha, \kappa} = \E[\Phi_z(\tau_\star(\kappa) G + b_\star(\kappa))] \\
= &~\Phi_z(z_\alpha) + \phi_z(z_\alpha)\E[(\tau_\star(\kappa) G + b_\star(\kappa) - z_\alpha)] + \frac{1}{2} \phi_z'(z_\alpha)\E[(\tau_\star(\kappa) G + b_\star(\kappa) - z_\alpha)^2] \\
&~ + \frac{1}{6} \E[\phi_z''(\xi) (\tau_\star(\kappa) G + b_\star(\kappa) - z_\alpha)^3]\\
= &~ \alpha + \phi_z(z_\alpha)(b_\star(\kappa) - z_\alpha) + \frac{1}{2}\phi_z'(z_\alpha) \tau_\star^2(\kappa) + o(\kappa) \\
=&~ \alpha + \paren{\phi_z(z_\alpha)\barb_0 + \frac{1}{2} \phi_z'(z_\alpha) \bartau_0^2 } \kappa + o(\kappa),
\end{aligned}
\]
where the last equality is by Lemma \ref{lem:local-linear-expansion} and by the boundedness of $\phi_z''$. This proves Eq. (\ref{eqn:coverage_expansion}) and thus Lemma~\ref{lem:coverage_expansion}.

\section{Extension to over-parametrized learning}
\label{appendix:overparam}


In this section we provide a variant of Theorem~\ref{theorem:main} in the over-parametrized case, i.e. when $d\ge n$, so that the learned quantile functions have the capacity to interpolate the entire training dataset. We still assume that the data are generated from the linear model~\eqref{equation:linear-model}. For notational simplicity, throughout this section we let $\btheta\defeq [\wb^\top, b]^\top\in\R^{d+1}$ denote the concatenation of $\wb$ and $b$, and let $\what{R}_n(\btheta)$ denote the empirical risk~\eqref{equation:quantile-reg-linear}. We also let $\wt{\xb}=[\xb^\top, 1]^\top\in\R^{d+1}$ denote the augmented feature so that $\btheta^\top\wt{\xb}=\wb^\top\xb+b$. We let $\wt{\Xb}\in\R^{n\times (d+1)}$ denote the augmented input matrix and $\zb\in\R^n$ denote the noise vector.

In the over-parametrized case, the ERM is no longer well-defined as there are multiple interpolating solutions. We consider instead the quantile functions obtained on the gradient descent path on the empirical risk $\what{R}_n$. More precisely, we consider the vanilla (sub)-gradient descent algorithm: Initialize $\btheta_1=\bzero$, and iterate for all $t\ge 1$
\begin{align}
  \label{equation:subgradient-descent}
  \btheta_{t+1} = \btheta_t - \eta_t \gb_t,
\end{align}
where $\gb_t\in\partial \what{R}_n(\btheta_t)$ is any sub-gradient of the empirical risk $\what{R}_n$~\eqref{equation:quantile-reg-linear} at $\btheta_t$.

\begin{theorem}[Quantile regression under over-parametrization]
  \label{theorem:overparam}
  Suppose the data is generated from the Gaussian linear model~\eqref{equation:linear-model} with $\ltwo{\wb}=R$, and the nominal quantile level $\alpha\in(0.5, 1)$. Further assume the noise distribution $P_z$ is symmetric about $0$ and $\sigma^2$-sub-Gaussian. Then, there exists an absolute constant $C_0>0$ such that if $n\ge C_0(d+\log(1/\delta))$, the following holds.

  Let $\btheta_t$ be the iterates of the sub-gradient descent algorithm~\eqref{equation:subgradient-descent} with step-size $\eta_t\defeq \beta/\sqrt{t}$ for any $\beta>0$, and let $\btheta_\infty\in\R^{d+1}$ denote any limit point of $\set{\btheta_t}_{t\ge 1}$, then we have
  \begin{enumerate}[label=(\alph*), leftmargin=1.5pc]
  \item $\btheta_\infty$ is the minimum $\ell_2$-norm interpolator of the training data, i.e.
    \begin{align*}
      \btheta_\infty = \argmin_{\btheta\in\R^d} \set{\ltwo{\btheta}: \wt{\Xb}\btheta = \yb}.
    \end{align*}
  \item With probability at least $1-\delta$ (over the training data), the coverage of the limiting quantile function $\hat{f}_\infty\defeq \btheta_\infty^\top \wt{\xb} = \wb_\infty^\top\xb + b_\infty$ concentrates around $0.5$:
    \begin{align*}
      \abs{\coverage(\hat{f}_\infty) - 0.5} \le C(R+\sigma) \cdot \sqrt{\frac{\log(1/\delta)}{d}} \le C(R+\sigma) \cdot \sqrt{\frac{\log(1/\delta)}{n}},
    \end{align*}
    where $C>0$ is a constant that only depends on $\sup_{t\in\R} |\phi_z(t)|$.
  \end{enumerate}
\end{theorem}

\paragraph{Implications}
Theorem~\ref{theorem:overparam} shows that a severe under-coverage bias in the over-parametrized case: The coverage of the limiting quantile function (of the gradient descent path) is $0.5\pm \wt{O}(1/\sqrt{d})$, \emph{regardless of the nominal quantile level} $\alpha\in(0.5,1)$. Therefore $\hat{f}_\infty$ under-covers by $\alpha-0.5=\Theta(1)$, and this under-coverage bias does not diminish as we increase $n,d$.

The proof of Theorem~\ref{theorem:overparam} is established in the following two subsections.

\subsection{Proof of Part (a)}
We begin by observing that the sub-gradients of the quantile risk~\eqref{equation:quantile-reg-linear} takes the form
\begin{align}
  \label{equation:subgradient-quantile-risk}
  \gb_t = \frac{1}{n}\sum_{i=1}^n (\ell^\alpha)'(y_i - \btheta_t^\top \wt{\xb}_i) \cdot \wt{\xb}_i \in {\rm span}\set{\wt{\xb}_1,\dots,\wt{\xb}_n},
\end{align}
where $(\ell^{\alpha})'(t)$ is the sub-gradient of $\ell^\alpha$, which takes value $-(1-\alpha)$ at $t<0$, $\alpha$ at $t>0$, and any value within $[-(1-\alpha), \alpha]$ at $t=0$. As we initialized at $\btheta_1=\bzero$, this implies that
\begin{align*}
  \btheta_t \in {\rm span}\set{\wt{\xb}_1,\dots,\wt{\xb}_n}
\end{align*}
for all $t\ge 1$. Also, by~\eqref{equation:subgradient-quantile-risk} we have $\ltwo{\gb_t}\le M\defeq \max_{i\in[n]} \ltwo{\wt{\xb}_i}$, since $|(\ell^\alpha)'|\le \max\set{\alpha, 1-\alpha}\le 1$.

Also, let $\btheta_{\ell_2}$ denote the minimum $\ell_2$-norm interpolator of the dataset:
\begin{align}
  \label{equation:min-l2-formula}
  \btheta_{\ell_2} \defeq \argmin_{\btheta\in\R^d} \set{\ltwo{\btheta}: \wt{\Xb}\btheta = \yb} = \wt{\Xb}^\dagger \yb = \wt{\Xb}^\top (\wt{\Xb}\wt{\Xb}^\top)^{-1}\yb.
\end{align}
This $\btheta_{\ell_2}$ exists whenever $d+1\ge n$ (so that $\wt{\xb}_i\in\R^{d+1}$ are linearly independent with probability one and thus $\wt{\Xb}\wt{\Xb}^\top\in\R^{n\times n}$ is invertible). It further satisfies
\begin{itemize}[leftmargin=1.5pc]
\item $\what{R}_n(\btheta_{\ell_2})=0$ (since $\btheta_{\ell_2}^\top\wt{\xb}_i=y_i)$. Therefore $\btheta_{\ell_2}$ is a minimizer of $\what{R}_n$ since $\what{R}_n\ge 0$.
\item $\btheta_{\ell_2}\in{\rm span}\set{\wt{\xb}_1,\dots,\wt{\xb}_n}$.
\item $\btheta_{\ell_2}$ is the only point within ${\rm span}\set{\wt{\xb}_1,\dots,\wt{\xb}_n}$ that satisfies $\what{R}_n(\btheta_{\ell_2})=0$, as any such point $\btheta\in\R^{d+1}$ must satisfy $\wt{\Xb}\btheta=\yb$, and there is only one such point in the span because of the linear independence of $\set{\wt{\xb}_i}_{i=1}^n$.
\end{itemize}
We now use the following lemma on the last-iterate convergence of sub-gradient descent, adapted from~\citep[Corollary 3]{orabona2020last}:
\begin{lemma}[Last-iterate convergence of sub-gradient descent]
  \label{lemma:subgradient}
  Suppose $F:\R^D\to\R$ is a convex function with bounded sub-gradients: $\ltwo{\gb}\le M$ for all $\gb\in\partial F(\btheta)$ and any $\btheta\in\R^D$. Let $\btheta_\star\in\R^D$ be any minimizer of $F$ with $F_\star=F(\btheta_\star)>-\infty$. Consider the sub-gradient descent algorithm
  \begin{align*}
    \btheta_{t+1} = \btheta_t - \eta_t \gb_t,
  \end{align*}
  where $\gb_t\in\partial F(\btheta_t)$, and $\eta_t=\beta/\sqrt{t}$ for some $\beta>0$. Then, we have for all $T\ge 3$ that
  \begin{align*}
    F(\btheta_T) - F_\star \le \frac{\ltwo{\btheta_1 - \btheta_\star}^2 + 4M^2\beta^2\log T}{2\beta\sqrt{T}}.
  \end{align*}
\end{lemma}
Applying Lemma~\ref{lemma:subgradient} with on the quantile risk $\what{R}_n$ the associated minimizer $\btheta_{\ell_2}$, we get that (for $T\ge 3$)
\begin{align*}
  \what{R}_n(\btheta_T) \le \frac{\ltwo{\btheta_{\ell_2}}^2 + 4M^2\beta^2\log T}{2\beta\sqrt{T}}.
\end{align*}
This implies that $\what{R}_n(\btheta_T)\to 0$ as $T\to\infty$.

The above implies that any limit point $\btheta_\infty$ of the sequence $\set{\btheta_t}_{t\ge 1}$ must satisfy
\begin{itemize}[leftmargin=1.5pc]
\item $\what{R}_n(\btheta_\infty)=0$, by continuity of $\what{R}_n$;
\item $\btheta_\infty\in{\rm span}(\wt{\xb}_1,\dots,\wt{\xb}_n)$, by the closedness of the span.
\end{itemize}
Combined with the above assertions on $\btheta_{\ell_2}$, this shows that $\btheta_\infty=\btheta_{\ell_2}$, establishing part (a) of the theorem.
\qed

\subsection{Proof of part (b)}

We first establish a covariance lower bound useful for the subsequent analyses.
As $\xb_i\sim\normal(\bzero, \Ib_d)$, the input matrix $\Xb\in\R^{n\times d}$ has i.i.d. $\normal(0,1)$ entries, and thus $\Xb$'s columns are also i.i.d. $\normal(\bzero, \Ib_n)$. By standard sub-Gaussian covariance concentration, we have with probability at least $1-\delta$ that
\begin{align*}
  \opnorm{\frac{1}{d}\Xb\Xb^\top - \Ib_n} \le C\paren{\sqrt{\frac{n+\log(1/\delta)}{d}} + \frac{n+\log(1/\delta)}{d} }
\end{align*}
for some absolute constant $C>0$ (this can be found in e.g.~\citep[Example 4.7.3]{vershynin2018high}). In particular, we have $\opnorm{\Xb\Xb^\top/d - \Ib_n} \le 1/4$ provided $d\ge C(n+\log(1/\delta))$. On this event, we have
\begin{align*}
  \Xb\Xb^\top \succeq \frac{3d}{4}\Ib_n.
\end{align*}
We will apply a small variant of this result: as long as $d-1\ge C(n+\log(1/\delta))$, we also have for any fixed matrix $\Vb_\star\in\R^{d\times (d-1)}$ with orthogonal columns that
\begin{align}
  \label{equation:cov-lower-bound}
  \Xb\Vb_\star\Vb_\star^\top\Xb^\top \succeq \frac{3(d-1)}{4}\Ib_n \succeq \frac{d}{2}\Ib_n.
\end{align}

\paragraph{Bounding $|b_\infty|$}
By~\eqref{equation:min-l2-formula}, we have
\begin{align*}
  & \quad \begin{bmatrix}
    \wb_\infty \\ b_\infty
  \end{bmatrix} = \btheta_\infty = \btheta_{\ell_2} = \wt{\Xb}^\top (\wt{\Xb}\wt{\Xb}^\top)^{-1}\yb = \wt{\Xb}^\top (\wt{\Xb}\wt{\Xb}^\top)^{-1} (\Xb\wb_\star + \zb) \\
  & = \begin{bmatrix} \Xb^\top \\ \ones_n^\top \end{bmatrix} (\wt{\Xb}\wt{\Xb}^\top)^{-1} (\Xb\wb_\star + \zb).
\end{align*}
Therefore
\begin{align*}
  b_\infty=\ones_n^\top (\wt{\Xb}\wt{\Xb}^\top)^{-1} (\Xb\wb_\star + \zb) = \underbrace{\ones_n^\top \paren{\Xb\Xb^\top + \ones_n\ones_n^\top}^{-1} \Xb\wb_\star}_{\rm I} + \underbrace{\ones_n^\top \paren{\Xb\Xb^\top + \ones_n\ones_n^\top}^{-1} \zb}_{\rm II}.
\end{align*}
We now bound terms I and II separately.

For term I, let us assume for the moment that $\ltwo{\wb_\star}=1$. Let $\Vb_\star\in\R^{d\times d-1}$ denote the orthogonal complement to the matrix $\wb_\star$ (i.e. so that $[\wb_\star,~ \Vb_\star]\in\R^{d\times d}$ is an orthogonal matrix). We have
\begin{align*}
  {\rm I} = \ones_n^\top \paren{\Xb\Vb_\star\Vb_\star^\top\Xb^\top + \Xb\wb_\star\wb_\star^\top\Xb^\top + \ones_n\ones_n^\top}^{-1}\Xb\wb_\star.
\end{align*}
As $\Xb\Vb_\star\Vb_\star^\top\Xb^\top$ is an positive definite matrix with probability one whenever $d-1\ge n$, applying Lemma~\ref{lemma:woodbury} twice, we get
\begin{align*}
  \abs{\rm I} \le \abs{ \ones_n^\top \paren{\Xb\Vb_\star\Vb_\star^\top\Xb^\top + \ones_n\ones_n^\top}^{-1}\Xb\wb_\star  } \le \abs{ \ones_n^\top \paren{\Xb\Vb_\star\Vb_\star^\top\Xb^\top}^{-1}\Xb\wb_\star }.
\end{align*}
Now, notice that $\Xb\Vb_\star\in\R^{n\times d-1}$ and $\Xb\wb_\star\in\R^n$ have i.i.d. $\normal(0,1)$ entries and are independent of each other. Further, $\Xb\wb_\star\sim\normal(\bzero, \Ib_n)$, and thus the random variable $\ones_n^\top \paren{\Xb\Vb_\star\Vb_\star^\top\Xb^\top}^{-1}\Xb\wb_\star$ (conditional on $\Xb \Vb_\star$) is $\ltwo{\vb_{\rm I}}^2$-sub-Gaussian (due to the independence between $\Xb \Vb_\star$ and $\Xb \wb_\star$), where
\begin{align*}
  \ltwo{\vb_{\rm I}}^2 = \ones_n^\top \paren{\Xb\Vb_\star\Vb_\star^\top\Xb^\top}^{-2}\ones_n \le \frac{4}{d^2}\ltwo{\ones_n}^2 = \frac{4n}{d^2},
\end{align*}
where the inequality used the covariance lower bound~\eqref{equation:cov-lower-bound}. This shows that
\begin{align*}
  \abs{\rm I} \le C\sqrt{4n/d^2 \cdot \log(1/\delta)} \le C\sqrt{\log(1/\delta)/d}
\end{align*}
with probability at least $1-\delta$, where the last step used $n\le d$. It is straightforward to see that, for general $\ltwo{\wb_\star}=R$, we have
\begin{align}
  \label{equation:overparam-term-i}
  \abs{\rm I} \le CR\sqrt{4n/d^2 \cdot \log(1/\delta)} \le CR\sqrt{\log(1/\delta)/d}.
\end{align}

For term II, As $\Xb$ and $\zb$ are independent, the random variable ${\rm II}=\ones_n^\top(\Xb\Xb^\top + \ones_n\ones_n^\top)^{-1}\zb$ (conditional on $\bX$) is $\ltwo{\vb_{\rm II}}^2\sigma^2$-sub-Gaussian, where
\begin{align*}
  \ltwo{\vb_{\rm II}}^2 = \ones_n^\top (\Xb\Xb^\top + \ones_n\ones_n^\top)^{-2}\ones_n \le \frac{4}{d^2}\ltwo{\ones_n}^2 = \frac{4n}{d^2} \le \frac{4}{d}.
\end{align*}
Similar as above, we have with probability at least $1-\delta$ that
\begin{align}
  \label{equation:overparam-term-ii}
  \abs{\rm II} \le C\sigma\sqrt{\log(1/\delta)/d}.
\end{align}
Combining~\eqref{equation:overparam-term-i} and~\eqref{equation:overparam-term-ii}, we get with probability at least $1-\delta$ that (rescaling $3\delta\to \delta$)
\begin{align}
  \label{equation:overparam-bound-b}
  \abs{b_\infty} \le C\paren{R+\sigma}\sqrt{\log(1/\delta)/d}.
\end{align}

\paragraph{Bounding the coverage bias}
We now translate the bound on $\abs{b_\infty}$ to a bound on the coverage error $\abs{\coverage(\hat{f}_\infty) - 0.5}$. First, note that by symmetry of the distribution of $(\wb_\infty-\wb_\star)^\top\xb$ and the fact that $\Phi_z(t)+\Phi_z(-t)=1$ (due to the symmetry of $P_z$), we have
\begin{align*}
  \E\brac{ \Phi_z\paren{ (\wb_\infty-\wb_\star)^\top\xb } } = \E\brac{ \frac{1}{2}\paren{ \Phi_z\paren{ (\wb_\infty-\wb_\star)^\top\xb } + \Phi_z\paren{ -(\wb_\infty-\wb_\star)^\top\xb } } } = 0.5.
\end{align*}
Therefore we have
\begin{align*}
  & \quad \abs{ \coverage(\hat{f}_\infty) - 0.5 } = \abs{ \E\brac{ \Phi_z\paren{ (\wb_\infty-\wb_\star)^\top\xb + b_\infty } -  \Phi_z\paren{ (\wb_\infty-\wb_\star)^\top\xb } } } \\
  & \le \sup_{t\in\R} |\phi_z(t)| \cdot \abs{b_\infty} \\
  & \le C \sup_{t\in\R} |\phi_z(t)| \cdot \abs{b_\infty} \le C \sup_{t\in\R} |\phi_z(t)| \cdot \paren{R+\sigma}\sqrt{\log(1/\delta)/d}.
\end{align*}
Notably the bound is also upper bounded by $C \sup_{t\in\R} |\phi_z(t)| \cdot (R+\sigma)\sqrt{\log(1/\delta)/n}$ as we assumed $d\ge n$. This proves part (b) of the theorem.
\qed


\section{Proofs for Section~\ref{section:extension}}

\subsection{Proof of Corollary~\ref{corollary:b}}
\label{appendix:proof-b}


First, part (a) is a direct consequence of Lemma~\ref{lem:local-linear-expansion} which was established within the proof of Theorem~\ref{theorem:main}.

We now prove part (b). We first show that $\barb_0<0$ for $P_z$ being any Gaussian distribution. We first observe that to determine the sign of $\barb_0$, it suffices to consider the standard Gaussian: The value of $\barb_0$ does not depend on the location parameter (since $\phi_z$ and $z_\alpha$ shifts together with a location shift). Also, scalings won't change the sign of $\barb_0$ (although it scales the numerator and the denominator by a different amount).

We next calculate $\barb_0$ for $P_z=\normal(0,1)$. We have $\phi_z'(z_\alpha) = -z_\alpha \phi_z(z_\alpha)$ for $\phi_z(t)=\exp(-t^2/2)/\sqrt{2\pi}$. Therefore the numerator of $\barb_0$ is
\begin{align*}
  -\alpha(1-\alpha)\phi_z'(z_\alpha) - (2\alpha-1)\phi_z^2(z_\alpha) = \paren{\alpha(1-\alpha)z_\alpha - (2\alpha-1)\phi_z(z_\alpha) }\phi_z(z_\alpha).
\end{align*}
Consider the change of variable $t\defeq z_\alpha$ so that $\alpha=\Phi_z(t)$. To show the above quantity is negative, it suffices to show that
\begin{align*}
  & \Phi(t) (1-\Phi(t)) t - (2\Phi(t) - 1) \phi(t) < 0 \\
  \Longleftrightarrow & \underbrace{\frac{t(1-\Phi(t))}{\phi(t)} - 2 + \frac{1}{\Phi(t)}}_{\defeq F(t)} < 0
\end{align*}
for all $t>0$, where $\Phi(t)=\Phi_z(t)$ is shorthand for the standard Gaussian CDF. To show this, we first observe that $F(0)=-2+1/\Phi(0)=0$, and further
\begin{align*}
  F'(t) = \frac{(1+t^2)(1-\Phi(t))}{\phi(t)} - t - \frac{\phi(t)}{\Phi(t)^2}.
\end{align*}
We can numerically check that $F'(t) < -0.03$ for $t\in[0,1]$, within which range we have $F(t)<-0.03t<0$. On the other hand, using the Gaussian CDF approximation bound
\begin{align*}
  1-\frac{1}{t^2} \le \frac{t(1-\Phi(t))}{\phi(t)} \le 1 - \frac{1}{t^2} + \frac{3}{t^4}~~~\textrm{for all}~t>0,
\end{align*}
we have
\begin{align*}
  & \quad F(t) \le 1 - \frac{1}{t^2} + \frac{3}{t^4} - 2 + \frac{1}{1 - (t^{-1}-t^{-3})\phi(t)} \\
  & \stackrel{(i)}{\le} -\frac{1}{t^2} + \frac{3}{t^4} + 2(t^{-1}-t^{-3})\phi(t) \le \frac{3+2t^3\phi(t)-t^2}{t^4} \stackrel{(ii)}{<} 0,
\end{align*}
where (i) happens when $(t^{-1}-t^{-3})\phi(t)<1/2$, which happens for all $t\ge 1$, and (ii) happens when $t\ge 2$. This shows that $F(t)<0$ for $t\ge 2$. For $t\in[1,2]$, one can check numerically that $F(t)<-0.1<0$. This shows $F(t)<0$ for all $t>0$, which establishes $\barb_0<0$ for $P_z=\normal(0,1)$, showing the first claim in part (b).

Next, for any $\alpha\in(0.5, 1)$, we show that there exists a noise distributions $\wt{P}_z$ for which $\barb_0>0$. Indeed, simply take any smooth density $\phi_z$ (such as standard Gaussian density), and modify $\phi_z$ locally around $z_\alpha$ into some new smooth density $\wt{\phi}_z$ such that both the new $\alpha$-quantile $\wt{z}_\alpha\approx z_\alpha$ and $\wt{\phi}_z(\wt{z}_\alpha)\approx \phi_z(z_\alpha)$ (with arbitrarily small differences), but $\wt{\phi}_z'(\wt{z}_\alpha)<0$ is negative with a high magnitude $|\wt{\phi}_z'(\wt{z}_\alpha)|$. Taking this magnitude high enough, we can always make $-\alpha(1-\alpha)\wt{\phi}_z'(\wt{z}_\alpha)-(2\alpha-1)\wt{\phi}_z(\wt{z}_\alpha)^2>0$, which gives $\barb_0>0$ for the noise distribution $\wt{P}_z$ defined by the density $\wt{\phi}_z$. This shows the second claim in part (b).


\subsection{Proof of Theorem~\ref{theorem:w}}
\label{appendix:proof-w}

For any $\hat{f}(\xb)=\hat{\wb}^\top\xb + b_\star$, the coverage can be expressed as
\begin{align*}
  & \qquad \coverage(\hat{f}) = \P\paren{y \le \hat{\wb}^\top \xb + b_\star} \stackrel{(i)}{=} \P\paren{ \mu_\star(\xb) + \sigma_\star(\xb) z \le \hat{\wb}^\top \xb + b_\star} \\
  & \stackrel{(ii)}{=} \P\paren{\sigma_\star(\xb)(z - z_\alpha) \le (\hat{\wb} - \wb_\star)^\top \xb} = \P\paren{ z \le z_\alpha + \frac{(\hat{\wb} - \wb_\star)^\top \xb}{\sigma_\star(\xb)}} \\
  & = \E\brac{ \Phi_z\paren{ z_\alpha + \frac{(\hat{\wb} - \wb_\star)^\top \xb}{\sigma_\star(\xb)} }}.
\end{align*}
Above, (i) used the data distribution assumption~\eqref{equation:relaxed-model}, and (ii) follows by subtracting both sides by $\mu_\star(\xb)+\sigma_\star(\xb)=\wb_\star^\top\xb+b_\star$ by the linear true quantile assumption~\eqref{equation:linear-true-quantile-relaxed}.

Now, by assumption $\alpha\ge 3/4$, we have $z_\alpha>z_{1/2}=0$. We claim the following holds for all $t\in\R$:
\begin{align}
  \label{equation:phi-concavity}
  \frac{1}{2}\paren{ \Phi_z(z_\alpha+t) + \Phi_z(z_\alpha-t) } \le \Phi_z(z_\alpha) - ct^2\indic{|t|\le z_\alpha},
\end{align}
where $c>0$ is a constant that only depends on $\Phi_z$ and $z_\alpha$. To see this, notice that $\Phi_z''(t)=\phi_z'(t)<0$ for $t>0$ and thus $\Phi_z$ is concave for $t\ge 0$. Further, $\Phi_z$ is $c$-strongly concave on $[z_\alpha/2, 3z_\alpha/2]$ for some $c>0$ as $\Phi_z''(t)=\phi_z'(t)$ is continuous and negative on this compact interval. This shows that
\begin{align*}
  \frac{1}{2}\paren{ \Phi_z(z_\alpha+t) + \Phi_z(z_\alpha-t) } \le \Phi_z(z_\alpha) - ct^2
\end{align*}
for $|t|\le z_\alpha/2$, and further by the concavity of $\Phi_z$ on $[0, 2z_\alpha]$ that
\begin{align*}
  \frac{1}{2}\paren{ \Phi_z(z_\alpha+t) + \Phi_z(z_\alpha-t) } \le \frac{1}{2}\paren{ \Phi_z(z_\alpha+t_0) + \Phi_z(z_\alpha-t_0) }  \le \Phi_z(z_\alpha) - ct_0^2 \le \Phi_z(z_\alpha) - ct^2/4
\end{align*}
for $|t|\in(z_\alpha/2, z_\alpha]$ (where $t_0\defeq z_\alpha/2$). This verifies claim~\eqref{equation:phi-concavity} for $|t|\le z_\alpha$. On the other hand, if $|t|\ge z_\alpha$, we have (taking $t>0$ w.l.o.g.) $\Phi_z(z_\alpha+t) \le 1$ always and $\Phi(z_\alpha-t) \le \Phi_z(0) = 1/2$. Therefore
\begin{align*}
  \frac{1}{2}\paren{ \Phi_z(z_\alpha+t) + \Phi_z(z_\alpha-t) } \le \frac{1}{2}\paren{1 + 1/2} = 3/4 \le \Phi(z_\alpha).
\end{align*}
This verifies claim~\eqref{equation:phi-concavity} for $|t|>z_\alpha$.

Now, note that $(\hat{\wb}-\wb_\star)^\top\xb / \sigma_\star(\xb)$ is symmetric about $0$ by our assumption that $\xb$ has a symmetric distribution and $\sigma_\star(\xb)=\sigma_\star(-\xb)$. Therefore, we can rewrite and upper bound the coverage using~\eqref{equation:phi-concavity}:
\begin{align*}
  & \qquad \coverage(\hat{f}) = \E\brac{ \frac{1}{2}\paren{ \Phi_z\paren{ z_\alpha + \frac{(\hat{\wb} - \wb_\star)^\top \xb}{\sigma_\star(\xb)} } + \Phi_z\paren{ z_\alpha - \frac{(\hat{\wb} - \wb_\star)^\top \xb}{\sigma_\star(\xb)} } } } \\
  & \stackrel{(i)}{\le} \E\brac{ \Phi_z(z_\alpha) - c\paren{ \frac{(\hat{\wb} - \wb_\star)^\top \xb}{\sigma_\star(\xb)} }^2 \indic{ \abs{\frac{(\hat{\wb} - \wb_\star)^\top \xb}{\sigma_\star(\xb)}} \le z_\alpha }  } \\
  & = \alpha - c\E\brac{ \paren{ \frac{(\hat{\wb} - \wb_\star)^\top \xb}{\sigma_\star(\xb)} }^2 \indic{ \abs{\frac{(\hat{\wb} - \wb_\star)^\top \xb}{\sigma_\star(\xb)}} \le z_\alpha } } \\
  & \stackrel{(ii)}{\le} \alpha - \frac{c}{\overline{\sigma}^2}\E\brac{ \paren{(\hat{\wb} - \wb_\star)^\top \xb}^2 \indic{ \abs{(\hat{\wb} - \wb_\star)^\top\xb} \le z_\alpha\underline{\sigma} } } \\
  & = \alpha - \frac{c}{\overline{\sigma}^2} \paren{ (\hat{\wb} - \wb_\star)^\top \E[\xb\xb^\top](\hat{\wb} - \wb_\star) - \E\brac{ \paren{(\hat{\wb} - \wb_\star)^\top \xb}^2 \indic{ \abs{(\hat{\wb} - \wb_\star)^\top\xb} > z_\alpha\underline{\sigma} } } } \\
  & \stackrel{(iii)}{\le} \alpha - \frac{c}{\overline{\sigma}^2}\Bigg( \underline{\gamma}\ltwo{\hat{\wb} - \wb_\star}^2 - \underbrace{\E\brac{ \paren{(\hat{\wb} - \wb_\star)^\top \xb}^2 \indic{ \abs{(\hat{\wb} - \wb_\star)^\top\xb} > z_\alpha\underline{\sigma} } }}_{(\star)} \Bigg).
\end{align*}
Above, (i) used~\eqref{equation:phi-concavity}; (ii) used the bound $\underline{\sigma} \le \sigma_\star(\xb)\le\overline{\sigma}$; (iii) used the covariance lower bound $\E[\xb\xb^\top]\succeq \underline{\gamma}\Ib_d$.
Further, letting $r\defeq \ltwo{\hat{\wb} - \wb_\star}$, the random variable $(\hat{\wb}-\wb_\star)^\top\xb$ (with randomness only in $\xb$) is $Kr^2$-sub-Gaussian, since $\xb$ is $K$-sub-Gaussian by our assumption. Therefore the term $(\star)$ can be further upper bounded as
\begin{align*}
  & \quad (\star) \le \paren{ \E\brac{\paren{(\hat{\wb} - \wb_\star)^\top\xb}^4} \cdot \P\paren{ \abs{(\hat{\wb} - \wb_\star)^\top\xb} > z_\alpha\underline{\sigma} } }^{1/2} \\
  & \le \paren{ CK^2r^4 \cdot  2\exp(-z_\alpha^2\underline{\sigma}^2/Kr^2) }^{1/2} \\
  & \le CKr^2 \cdot \exp(-z_\alpha^2\underline{\sigma}^2/2Kr^2) \stackrel{(i)}{\le} \frac{1}{2}\underline{\gamma}r^2,
\end{align*}
where (i) happens if $r\le r_0$ for some $r_0=r_0(\underline{\gamma}, \underline{\sigma}, K, z_\alpha)$. Plugging this back into the preceding bound yields
\begin{align*}
  \coverage(\hat{f}) \le \alpha - \frac{c\underline{\gamma}}{2\overline{\sigma}^2} \cdot r^2 = \alpha - \frac{c\underline{\gamma}}{2\overline{\sigma}^2} \cdot \ltwo{\hat{\wb} - \wb_\star}^2
\end{align*}
for any $\hat{\wb}$ such that $\ltwo{\hat{\wb} - \wb_\star}\le r_0$. This proves the desired result.
\qed
\section{Additional experimental details and ablations}
\label{appendix:experiment}

\subsection{Simulations}
\label{appendix:simulations}

We provide additional details about our simulations in Section~\ref{section:simulations}. In each problem instance, we generate $(\xb_i, y_i)$ from the Gaussian linear model~\eqref{equation:linear-model}: $\xb_i\sim\normal(\bzero, \Ib_d)$, $y_i=\wb_\star^\top\xb_i + z_i$ where $z_i\simiid P_z=\normal(0, 0.25)$. We choose $\ltwo{\wb_\star}=1$. We run the (sub)-gradient descent algorithm on the full empirical risk $\what{R}_n$ (note the risk also depends on the quantile level $\alpha$) for 50k steps, with initial learning rate $0.01$ and a 10x learning rate decay at the 25k-th step. For all our settings (choice of $n,d,\alpha$), this optimization schedule ensures that the training loss changes by less than $10^{-5}$ between consecutive iterations at the final iteration.

Each problem instance yields a solution $(\hat{\wb}, \hat{b})$ which specifies a linear quantile function $\hat{f}(\xb)=\hat{\wb}^\top\xb + \hat{b}$. We evaluate its coverage \emph{exactly} using the closed-form formula (cf. Section~\ref{section:proof-sketch})
\begin{align*}
  \coverage(\hat{f}) = \E_{G\sim\normal(0,1)}\brac{ \Phi_z(\ltwo{\hat{\wb} - \wb_\star}G + \hat{b}) }.
\end{align*}
We compute this by using numerical integration (over the gaussian random variable $G$). The entire set of experiments (for producing Figure~\ref{figure:sim}) is done on a single CPU machine in roughly 6 hours.

\subsection{Real data experiments}
\label{appendix:real}
We provide additional details about our real data experiments in Section~\ref{section:real} and~\ref{section:pseudo}. All models (linear, MLP, MLP-freeze) in Section~\ref{section:real} are trained by minimizing the quantile risk~\eqref{eqn:quantile_regression_f_theta}. We use SGD with momentum 0.9, initial learning rate $10^{-3}$ for 1500 epochs, and apply a 10x learning rate decay at epoch $\set{500, 1000}$. For each dataset and each random seed, we perform a train-validation split where we use $80\%$ of the data as the train set and $20\%$ of the data as the test set. The coverage of the trained model is evaluated on the test split. For all datasets and all models, we repeat the same experiment across 8 random seeds, and report the mean and standard deviation of the coverage in Table~\ref{table:real}.

For our pseudo-label experiments in Section~\ref{section:pseudo}, we train the linear model $\what{\wb}$ first by minimizing the square loss and using the same optimization schedule above. After $\what{\wb}$ is learned, we generate the pseudo-labels $y_i^{\rm pseudo}$ using $\what{\wb}$ and the estimated standard deviation $\what{\sigma}$ as described in Section~\ref{section:pseudo}. This is done for both the train and test sets for which we obtain a ``pseudo'' train set and a ``pseudo'' test set. We then perform linear quantile regression on these pseudo datasets in a same fashion as in Section~\ref{section:pseudo}.

The experiments for Sections~\ref{section:real} and~\ref{section:pseudo} are done on a 8-GPU machine (with Tesla V-100 GPUs) in roughly a day.

\paragraph{Ablations on $\alpha$}
Table~\ref{table:real-alpha-080} and~\ref{table:real-alpha-095} report coverage results on the real data with $\alpha\in\set{0.8, 0.95}$ respectively, in the same settings as in Section~\ref{section:real}. These tables also show that under-coverage happens consistently across different datasets and different models, with patterns similar as in Table~\ref{table:real} (which uses $\alpha=0.9$).

\begin{table}[th]
  \small
  \caption{Coverage ($\%$) of quantile regression on real data at nominal level $\alpha=0.8$. Each entry reports the test-set coverage with mean and std over 8 random seeds. $(d,n)$ denotes the \{feature dim, \# training examples\}.}
  \label{table:real-alpha-080}
  \centerline{
    \begin{tabular}{l|llll|rr}
      \toprule
      Dataset & Linear & MLP-3-64 & MLP-3-512 & MLP-freeze-3-512 & $d$ & $n$ \\
      \midrule
      Community & 78.25$\pm$1.75 & 66.07$\pm$1.48 & 56.17$\pm$2.81 & 77.45$\pm$1.76 & 100 & 1599 \\
      Bike & 79.95$\pm$0.66 & 78.07$\pm$1.00 & 78.66$\pm$0.86 & 79.46$\pm$0.83 & 18 & 8708 \\
      Star &79.97$\pm$2.37 & 72.95$\pm$1.83 & 59.26$\pm$1.41 & 78.42$\pm$2.04 & 39 & 1728 \\
      MEPS\_19 & 80.11$\pm$1.12 & 76.47$\pm$0.93 & 70.04$\pm$0.75 &79.02$\pm$1.28 & 139 & 12628 \\
      MEPS\_20 & 79.84$\pm$0.75 & 77.11$\pm$0.73 & 71.88$\pm$0.87 & 79.29$\pm$0.53 & 139 & 14032 \\
      MEPS\_21 & 79.57$\pm$0.72 & 74.58$\pm$0.70 & 65.55$\pm$0.69 &79.29$\pm$0.73 & 139 & 12524 \\
      \midrule
      Nominal ($\alpha$) & 80.00 & 80.00 & 80.00 & 80.00 & - & - \\
      \bottomrule
    \end{tabular}
  }
  \vspace{-1em}
\end{table}

\begin{table}[th]
  \small
  \caption{Coverage ($\%$) of quantile regression on real data at nominal level $\alpha=0.95$. Each entry reports the test-set coverage with mean and std over 8 random seeds. $(d,n)$ denotes the \{feature dim, \# training examples\}.}
  \label{table:real-alpha-095}
  \centerline{
    \begin{tabular}{l|llll|rr}
      \toprule
      Dataset & Linear & MLP-3-64 & MLP-3-512 & MLP-freeze-3-512 & $d$ & $n$ \\
      \midrule
      Community & 93.82$\pm$0.98 & 86.23$\pm$1.43 & 74.38$\pm$1.86 & 93.58$\pm$1.33 & 100 & 1599 \\
      Bike & 94.56$\pm$0.45 & 93.77$\pm$0.63 & 93.16$\pm$0.80 & 94.19$\pm$0.65 & 18 & 8708 \\
      Star & 94.08$\pm$1.73 & 90.96$\pm$1.91 & 81.58$\pm$1.82 & 93.39$\pm$1.68 & 39 & 1728 \\
      MEPS\_19 & 94.69$\pm$0.41 & 90.71$\pm$0.72 & 85.32$\pm$1.23 & 94.19$\pm$0.42 & 139 & 12628 \\
      MEPS\_20 & 94.84$\pm$0.30 & 92.06$\pm$0.43 & 87.32$\pm$0.77 & 94.58$\pm$0.32 & 139 & 14032 \\
      MEPS\_21 & 94.97$\pm$0.34 & 89.55$\pm$0.39 & 80.70$\pm$0.79 & 94.42$\pm$0.29 & 139 & 12524 \\
      \midrule
      Nominal ($\alpha$) & 95.00 & 95.00 & 95.00 & 95.00 & - & - \\
      \bottomrule
    \end{tabular}
  }
  \vspace{-1em}
\end{table}

\subsection{License of datasets}
The {\tt Community}~\citep{community} and {\tt Bike}~\citep{bike} datasets are retrieved from the publicly available UCI machine learning repository~\citep{dua2017} and subject to the license of the repository. The {\tt STAR} dataset~\citep{star} is also a public access dataset. The three mediecal expenditure survey datasets {\tt MEPS\_19}, {\tt MEPS\_20}, {\tt MEPS\_21} contain a data use agreement section in their documentation (cf. the ``documentation'' link in ~\citep{meps19,meps20,meps21}) which our use case (train quantile functions and report coverages) comply with. All the datasets are anonymized and to the best of our knowledge do not contain personally identifiable information or offensive contents.

\end{document}